\documentclass{article}

   \usepackage[preprint]{neurips_2020}
\usepackage[utf8]{inputenc} 
\usepackage[T1]{fontenc}    
\usepackage{hyperref}       
\usepackage{url}            
\usepackage{booktabs}       
\usepackage{color}
\usepackage{amsfonts}       
\usepackage{nicefrac}       
\usepackage{microtype}      
\usepackage{appendix}
\usepackage{hyperref}
\usepackage{amsthm}
\usepackage{amsmath}
\usepackage{amssymb}
\usepackage{amsfonts}
\usepackage{enumerate}
\usepackage{algorithm}
\usepackage{algorithmic} 
\usepackage{geometry}
\usepackage{xspace}
\usepackage{mathabx}
\usepackage{graphicx}
\usepackage{eufrak}
\usepackage{scalerel,stackengine}
\stackMath
\newcommand\reallywidehat[1]{%
\savestack{\tmpbox}{\stretchto{%
  \scaleto{%
    \scalerel*[\widthof{\ensuremath{#1}}]{\kern-.6pt\bigwedge\kern-.6pt}%
    {\rule[-\textheight/2]{1ex}{\textheight}}
  }{\textheight}%
}{0.5ex}}%
\stackon[1pt]{#1}{\tmpbox}%
}
\newtheorem{question}{Question}
\newtheorem{definition}{Definition}

\newtheorem{thm}{Theorem}
\newtheorem{lemma}[thm]{Lemma}
\newtheorem{proposition}[thm]{Proposition}

\newtheorem{corollary}[thm]{Corollary}

\newcommand{\UCBADV}{{\sc UCB-Advantage}\xspace}

\def\udl#1{\underline{#1} }

\title{Almost Optimal Model-Free Reinforcement Learning via Reference-Advantage Decomposition}
\author{%
	Zihan Zhang \\
	Department of Automation\\
	Tsinghua University\\
	\texttt{zihan-zh17@mails.tsinghua.edu.cn} \\
	\AND
	Yuan Zhou \\
	Department of ISE\\
	University of Illinois at Urbana-Champaign \\
	\texttt{yuanz@illinois.edu} \\
	\And
	Xiangyang Ji \\
	Department of Automation\\
	Tsinghua University \\
	\texttt{xyji@tsinghua.edu.cn} \\
}
\begin{document}
	\maketitle
\begin{abstract}

We study the reinforcement learning problem in the setting of finite-horizon episodic Markov Decision Processes (MDPs) with $S$ states, $A$ actions, and episode length $H$. We propose a model-free algorithm \UCBADV and prove that it achieves $\tilde{O}(\sqrt{H^2SAT})$ regret where $T = KH$ and $K$ is the number of episodes to play. Our regret bound improves upon the results of \citep{jin2018q} and matches the best known model-based algorithms as well as the information theoretic lower bound up to logarithmic factors. We also show that \UCBADV achieves low local switching cost and applies to concurrent reinforcement learning, improving upon the recent results of \citep{bai2019provably}.

\end{abstract}

\allowdisplaybreaks

\section{Introduction}
Reinforcement learning (RL) \citep{Burnetas1997Optimal} studies the problem where an agent aims to maximize its accumulative rewards through sequential decision making in an unknown environment modeled by Markov Decision Processes (MDPs). At each time step, the agent observes the current state $s$ and  interacts with the environment by taking an action $a$ and transits to next state $s'$ following the underlying transition model.

There are mainly two types of algorithms to approach reinforcement learning: \emph{model-based} and \emph{model-free} learning.  Model-based algorithms learn a model from the past experience and make decision based on this model while  model-free algorithms only maintain a group of value functions and take the induced optimal actions. 
Because of these differences, model-free algorithms are usually more space- and time-efficient compared to model-based algorithms. Moreover, because of their simplicity and flexibility, model-free algorithms are popular in a wide range of practical tasks (e.g., DQN \citep{mnih2015human}, A3C \citep{mnih2016asynchronous}, TRPO \citep{schulman2015trust}, and PPO \citep{schulman2017proximal}). On the other hand, however, it is believed that model-based algorithms may be able to take the advantage of the learned model and achieve better learning performance in terms of regret or sample complexity, which has been empirically evidenced by \cite{deisenroth2011pilco} and \cite{schulman2015trust}.
Much experimental research has been done for both types of the algorithms, and given that there has been a long debate on their pros and cons that dates back to \citep{deisenroth2011pilco}, a natural and intriguing theoretical question to study about reinforcement learning algorithms is that --
\begin{question}\label{question:main}
Is it possible that model-free algorithms achieve as competitive learning efficiency as  model-based algorithms, while still maintaining low time and space complexities?
\end{question}

Towards answering this question, the recent work by \cite{jin2018q} formally defines that an RL algorithm is model-free if its space complexity is always sublinear relative to the space required to store the MDP parameters, and then proposes a model-free algorithm (which is a variant of the $Q$-learning algorithm \citep{watkins1989learning}) that achieves the first $\sqrt{T}$-type regret bound for finite-horizon episodic MDPs in the tabular setting (i.e., discrete state spaces). However, there is still a gap of factor $\sqrt{H}$ between the regret of their algorithm and the best model-based algorithms. In this work, we close this gap by proposing a novel model-free algorithm, whose regret matches the optimal model-based algorithms, as well as the information theoretic lower bound. The results suggest that model-free algorithms can learn as efficiently as model-based ones, giving an affirmative answer to Question~\ref{question:main} in the setting of episodic tabular MDPs.

\subsection{Our Results}
\vspace{-1ex}
\paragraph{Main Theorem.} We propose a novel variant of the $Q$-learning algorithm, \UCBADV. We then prove the following main theorem of the paper.

\begin{thm}\label{thm1}  For $T$ greater than some polynomial of $S$, $A$, and $H$, and for any $p \in (0, 1)$, with probability $(1-p)$, the regret of \UCBADV is bounded by
	$\mathrm{Regret}(T)\leq \tilde{O}(\sqrt{H^2 SAT})$,
where poly-logarithmic factors of $T$ and $1/p$ are hidden in the $\tilde{O}(\cdot)$ notation.
\end{thm}
Compared to the $\tilde{O}(\sqrt{H^3SA T})$ regret bound of the UCB-Bernstein algorithm in \citep{jin2018q}, \UCBADV saves a factor of $\sqrt{H}$, and matches the information theoretic lower bound of $\Omega(\sqrt{H^2 SAT})$ in \citep{jin2018q} up to logarithmic factors. The regret of \UCBADV is at the same order of the best model-based algorithms such as UCBVI \citep{azar2017minimax} and vUCQ \citep{kakade2018variance}.\footnote{Both \cite{azar2017minimax} and \cite{kakade2018variance} assume equal transition matrices $P_1 = P_2 = \dots = P_H$. In this work, we adopt the same setting as in, e.g., \citep{jin2018q} and \citep{bai2019provably}, where $P_1, P_2, \dots, P_H$ can be different. This adds a factor of $\sqrt{H}$ to the regret analysis in \citep{azar2017minimax} and \citep{kakade2018variance}.} However, the time complexity before time step $T$ is $O(T)$ and the space complexity is $O(SAH)$ for \UCBADV. In contrast, both UCBVI and vUCQ uses $\tilde{O}(TS^2A)$ time and $O(S^2AH)$ space.

One of the main technical ingredients of  \UCBADV is to incorporate a novel update rule for the $Q$-function based on the proposed \emph{reference-advantage decomposition}. More specifically, we propose to view the optimal value function $V^*$ as $V^* = V^{\mathrm{ref}} + (V^* - V^{\mathrm{ref}})$, where $V^{\mathrm{ref}}$, the \emph{reference} component, is a comparably easier learned approximate of $V^*$ and  the other component $(V^* - V^{\mathrm{ref}})$ is referred to as the \emph{advantage} part. Based on this decomposition, the new update rule learns the corresponding parts of the $Q$-function using carefully designed (and different) subsets of the collected data, so as to minimize the deviation, maximize the data utilization, and reduce the estimation variance.

Another highlight of  \UCBADV  is the use of the \emph{stage-based update framework} which enables an easy integration of the new update rule (as above) and the standard update rule. In such a framework, the visits to each state-action pair are partitioned into \emph{stages}, which are used to design the trigger and subsets of data for each update. 

\paragraph{Implications.} An extra benefit of the stage-based update framework is to ensure the low frequency of policy switches of \UCBADV, stated as follows.
\begin{thm}\label{thm:sw-cost}
The local switching cost of \UCBADV is bounded by $O(SAH^2\log T)$.
\end{thm}
While one may refer to Appendix~\ref{App.C} for the details of the theorem, the notion of local switching cost for RL is recently introduced and studied by \cite{bai2019provably}, where the authors integrate a lazy update scheme with the UCB-Bernstein algorithm \citep{jin2018q} and achieve $\tilde{O}(\sqrt{H^3 SAT})$ regret and $O(SAH^3\log T)$ local switching cost. In contrast, our result improves in both metrics of regret and switching cost.

Our results also apply to concurrent RL, a research direction closely related to batched learning and learning with low switching costs, stated as follows.

\begin{corollary} \label{cor:concurrent-RL}
Given $M$ parallel machines,  the concurrent and pure exploration version of \UCBADV can compute an $\epsilon$-optimal policy in 
$    \tilde{O}(H^2SA+ H^3SA/(\epsilon^2 M))$
concurrent episodes.
\end{corollary}
In contrast, the state-of-the-art result \citep{bai2019provably} uses $\tilde{O}(H^3SA+ H^4SA/(\epsilon^2 M))$ concurrent episodes.  When $M = 1$, Corollary~\ref{cor:concurrent-RL} implies that the single-threaded exploration version of \UCBADV uses $\tilde{O}(H^3SA/\epsilon^2)$ episodes to learn an $\epsilon$-optimal policy. In Appendix~\ref{App.C}, we provide a simple $\Omega(H^3SA/\epsilon^2)$-episode lower bound for the sample complexity, showing the optimality up to logarithmic factors.

\vspace{-1ex}
\subsection{Additional Related Works}\label{sec:related-works}

\paragraph{Regret Analysis for RL.} Since our results focus on the tabular case, we will not mention most of the results on RL for continuous state spaces. For the tabular setting, there are plenty of recent works on model-based algorithms under various settings (e.g., \citep{jaksch2010near,agrawal2017optimistic,azar2017minimax,ouyang2017learning,fruit2019improved,simchowitz2019non,zanette2019tighter,zhang2019regret}). The readers may refer to \citep{jin2018q} for more detailed review and comparison. In contrast, fewer model-free algorithms are proposed. Besides \citep{jin2018q}, an earlier work \citep{strehl2006pac} implies that $T^{4/5}$-type regret can be achieved by a model-free algorithm.  

\paragraph{Variance Reduction and Advantage Functions.}  Variance reduction techniques via reference-advantage decomposition is used for faster optimization algorithms \citep{johnson2013accelerating}. The technique is also recently applied to pure exploration in learning discounted MDPs \citep{sidford2018variance,sidford2018near}. However, since \cite{sidford2018variance,sidford2018near} assume the access to a simulator and \UCBADV is completely online, our update rule and data partition design is very different. Our work is also the first for regret analysis in RL.

The use of advantage functions have also witnessed much success for RL in practice. For example, in A3C \citep{mnih2016asynchronous}, the advantage function is defined to be $\mathrm{Adv}(s, a) := Q^{\pi}(s, a) - V^{\pi}(s)$, and helps to reduce the estimation variance of the policy gradient. Similar definitions can also be found in other works such as \citep{sutton2000policy}, Generalized Advantage Estimation \citep{schulman2015high} and Dueling DQN \citep{wang2015dueling}. In comparison, our advantage function is defined on the states instead of the state-action pairs.

\section{Preliminaries}
We study the setting of episodic MDPs where an MDP is described by $(\mathcal{S},\mathcal{A},H, P,r)$. Here, $\mathcal{S}\times\mathcal{A}$ is the state-action space, $H$ is the length of each episode, $P$ is the  transition probability matrix  and $r$ is the deterministic reward function\footnote{Our results generalize to stochastic reward functions easily.}. Without loss of generality, we assume that $r_{h}(s,a)\in [0,1]$ for all $s,a,h$. During each episode, the agent observes the initial state $s_{1}$ which may be chosen by an \emph{oblivious adversary} (i.e., the adversary may have the access to the algorithm description used by the agent but does not observe the execution trajectories of the agent\footnote{Another adversary model is the the stronger \emph{adaptive adversary} who may observe the execution trajectories and select the initial states based on the observation. While it is possible that a more careful analysis of our algorithm also works for the adaptive adversary, we do not make any effort verifying this statement. We also note that previous works such as \citep{jin2018q,bai2019provably} do not explicitly define their adversary models and it is not clear whether their analysis works for the adaptive adversary.}). 

During each step within the episode, the agent takes an action $a_{h}$ and transits to $s_{h+1}$ according to $P_{h}(\cdot|s_{h},a_{h})$. The agent keeps running for $H$ steps and then the episode terminates. 

A policy\footnote{In this work, we mainly consider deterministic policies since the optimal value function can be achieved by a deterministic policy.} $\pi$ is a mapping from $\mathcal{S}\times[H]$ to $\mathcal{A}$. Given a policy $\pi$, we define its value function and $Q$-function as
\vspace{-2ex}
\begin{equation*}
\begin{aligned}
&V^{\pi}_{h}(s) = \mathbb{E} \left[  \sum_{h'=h}^{H} r_{h'}(s_{h'},\pi_{h'}(s_{h'})) \Big| s_{h}=s ,s_{h'+1}\sim P_{h'}(\cdot|s_{h'},\pi_{h'}(s_{h'})) \right], \\
& Q^{\pi}_{h}(s,a) = r_{h}(s,a)+ P_{h}(\cdot|s,a)^{\top}V^{\pi}_{h+1} = r_h(s, a) + P_{s,a,h} V^{\pi}_{h+1}.
\end{aligned}
\end{equation*}
As boundary conditions, we define $V^{\pi}_{H+1}(s)= Q^{\pi}_{H+1}(s,a)=0$ for any $\pi, s,a$. Also note that, for simplicity, throughout the paper, we use $xy$ to denote $x^T y$ for two vectors of the same dimension and use $P_{s,a,h}$ to denote $P_{h}(\cdot|s,a)$.

The optimal value function is then given by $V^*_{h}(s)=\sup_{\pi}V^{\pi}_{h}(s)$ and $Q^*_{h}(s,a)=r_{h}(s,a)+ P_{s,a,h}V^{*}_{h+1}$ for any $(s,a)\in \mathcal{S}\times\mathcal{A}, h\in [H]$.  

The learning problem consists of $K$ episodes, i.e, $T=KH$ steps. Let $s_1^k$ be the state given to the agent at the beginning of the $k$-th episode, and let $\pi_{k}$ be the policy adopted by the agent during the $k$-th episode. To goal is to minimize the total regret at time step $T$ which is defined as follows,
\begin{equation}
\mathrm{Regret}(T): = \sum_{k=1}^{K}\big(V^{*}_{1}(s^k_{1})-V^{\pi_{k}}_{1}(s^k_{1})\big).
\end{equation}

\section{The \UCBADV Algorithm}

In this section, we introduce the \UCBADV algorithm. We start by reviewing the $Q$-learning algorithms proposed in \citep{jin2018q}. Recall that \cite{jin2018q} selects the learning rate $\alpha_{t} = \frac{H+1}{H+t}$, and sets the weights $\alpha_{t}^i = \alpha_{i}\Pi_{j=i+1}^{t}(1-\alpha_{j})$ for the $i$-th samples out of the a total of $t$ data points, for any state-action pair. Note that $\alpha_t^i$ is roughly $\Theta(H/t)$ for the indices $i \in  [\frac{H-1}{H} \cdot t, t]$ and 
vanishes quickly when $i \ll \frac{H-1}{H} \cdot t$. As a result, their update process is 
 roughly equivalent to using the latest $\frac{1}{H}$ fraction of samples to update the value function for any state-action pair. Next, we introduce our stage-based update framework, which shares much similarity with the process discussed above. However, our framework enjoys simpler analysis and enables easier integration of the two update rules which will be explained afterwards.
 
\paragraph{Stages and Stage-Based Update Framework.} For any triple $(s, a, h)$, we divide the samples received for the triple into consecutive \emph{stages}. The length of each stage roughly increases exponentially with the growth rate $(1 + 1/H)$. More specifically, we define  $e_{1} = H$ and $e_{i+1} = \left \lfloor (1+\frac{1}{H})e_{i} \right \rfloor $ for all $i\geq 1$, standing for the length of the stages. We also let $\mathcal{L} :=\{\sum_{i=1}^{j}e_{i}| j = 1, 2, 3, \dots \}$ be the set of indices marking the ends of the stages.

Now we introduce the \emph{stage-based update framework}. For any $(s,a,h)$ triple, we update $Q_{h}(s,a)$ when the total visit number of $(s,a,h)$ the end of the current stage (in other word, the total visit number occurs in $\mathcal L$). Only the samples in the latest stage will be used in this update. Using the language of \citep{jin2018q}, for any total visit number $t$ in the $(j+1)$-th stage,
our update framework is equivalent to setting the weight distribution to be $\alpha_{t}^i = e_j^{-1} \cdot  \mathbb{I}\left[ i \text{ in the } j \text{-th stage}   \right]$.

We note that the definition of stages is with respect to the triple $(s, a, h)$. For any fixed pair of $k$ and $h$, let $(s_h^k, a_h^k)$ be the state-action pair at the $h$-th step during the $k$-th episode of the algorithm. We say that $(k,h)$ falls in the $j$-th stage of $(s,a,h)$ if and only if $(s,a) = (s^k_h,a^k_h)$ and the total visit number of $(s_h^k,a_h^k)$ after the $k$-th episode is in $(\sum_{i=1}^{j-1}e_i, \sum_{i=1}^j e_{i}]$.  

One benefit of our stage-based update framework is that it helps to reduce the number of the updates to the $Q$-function, leading to less local switching costs, which is recently also studied by \cite{bai2019provably}, where the authors propose to apply a lazy update scheme to the algorithms in \cite{jin2018q}. The lazy update scheme uses an exponential triggering sequence with a growth rate of $(1+1/(2H(H+1)))$, which is more conservative than the growth rate of stage lengths in our work. As a result, our algorithm saves an $H$ factor in the switching cost compared to \citep{bai2019provably}.

More importantly, our stage-based update framework, compared to the algorithms in \citep{jin2018q}, (in our opinion) simplifies the analysis, makes it easier to integrate the standard update rule and the one based on the reference-advantage decomposition. Both update rules are used in our algorithm, and we now discuss them separately.

\paragraph{The Standard Update Rule and its Limitation.} The algorithms in \citep{jin2018q} uses the following standard update rule,
  \begin{equation}\label{eqex0}
Q_{h}(s,a)  \leftarrow  \reallywidehat{P_{s,a,h} V_{h+1}} + r_{h}(s,a) + \overline{b},
 \end{equation}
 where $\overline{b}$ is the exploration bonus, and $\reallywidehat{P_{s,a,h} V_{h+1}}$ is the empirical estimate of $P_{s,a,h} V_{h+1}$. We also adopt this update rule in our algorithm. However, a crucial restriction is that the earlier samples collected, the more deviation one would expect between the $V_{h+1}$ learned at that moment and the true value. To ensure that these deviations do not ruin the whole estimate, we have to require that $\reallywidehat{P_{s,a,h} V_{h+1}}$ only uses the samples acquired from the last stage. This means that we can only estimate the $P_{s, a, h} V_{h+1}$ term using about $1/H$ fraction of the obtained data, and we note that this is also the reason of the extra $\sqrt{H}$ occurred in the UCB-Bernstein algorithm by \cite{jin2018q}.

\paragraph{Reference-Advantage Decomposition and the Advantage-Based Update Rule.} We now introduce the reference-advantage decomposition, which is the key to reducing the extra $\sqrt{H}$ factor. At a high level, we aim at first learning a quite accurate estimation of the optimal value function $V^*$ and denote it by the \emph{reference value function} $V^{\mathrm{ref}}$. The accuracy is controlled by an error parameter $\beta$ which is quite small but independent of $T$ or $K$. In other words, we wish to have $V^*_{h}(s)\leq V^{\mathrm{ref}}_{h}(s)\leq V^*_{h}(s)+ \beta$ for all $s$ and $h$, and for the purpose of simple explanation, we set $\beta = 1 / H$ at this moment; in our algorithm, $\beta$ can be any value that is less than $\sqrt{1/H}$ while independent of $T$ or $K$. 

For starters, let us first assume that we have the access to the dreamed $V^{\mathrm{ref}}$ reference function as stated above. Now we write $V^* = V^{\mathrm{ref}} + (V^* -  V^{\mathrm{ref}})$, and refer to the second term as the \emph{advantage} compared to the reference values\footnote{Interestingly, one might argue that the term should rather be called ``disadvantage'' as it is always non-positive. We choose the name ``advantage'' to highlight the similarity between our algorithm and many empirical algorithms in literature. See Section~\ref{sec:related-works} for more discussion.}. Now the $Q$-function can be updated using the following advantage-based rule,
  \begin{equation}\label{eqex}
Q_{h}(s,a)  \leftarrow  \reallywidehat{P_{s,a,h} V^{\mathrm{ref}}_{h+1}} + \reallywidehat{P_{s,a,h}(V_{h+1}- V^{\mathrm{ref}}_{h+1} )} + r_{h}(s,a) + b,
 \end{equation} 
where $b$ is the exploration bonus, and both $\reallywidehat{P_{s,a,h} V^{\mathrm{ref}}_{h+1}}$ and $\reallywidehat{P_{s,a,h}(V_{h+1}- V^{\mathrm{ref}}_{h+1} )}$ are empirical estimates of $P_{s,a,h} V^{\mathrm{ref}}_{h+1}$ and $P_{s,a,h}(V_{h+1}- V^{\mathrm{ref}}_{h+1} )$ (respectively)  based on the observed samples. We still have to require that $\reallywidehat{P_{s,a,h}(V_{h+1}- V^{\mathrm{ref}}_{h+1} )}$ uses the samples only from the last stage so as to limit the deviation error due to $V_{h+1}$ in the earlier samples.

Fortunately, thanks to the reference-advantage decomposition, and since that $V$ is learned based on $V^{\mathrm{ref}}$ and approximates $V^*$ even better than $V^{\mathrm{ref}}$, we have that $\| V_{h+1}-V^{\mathrm{ref}}_{h+1}\|_{\infty}\leq \beta = 1/H$ holds for all samples, which suffices to offset the weakness of using only $1/H$ of the total data, and helps to learn an accurate estimation of the second term. On the other hand, for the first term in the Right-Hand-Side of \eqref{eqex}, since $V^{\mathrm{ref}}$ is fixed and never changes, we are able to use all the samples collected to conduct the estimation, without suffering any deviation. This means that the first term can also be estimated with high accuracy.

The discussion till now has assumed that the reference value vector $V^{\mathrm{ref}}$ is known. To remove this assumption, we note that $\beta$ is independent of $T$, therefore a natural hope is to learn $V^{\mathrm{ref}}$ using sample complexity also almost independent of $T$, incurring regret only in the lower order terms. However, since it is not always possible to learn the value function of every state (especially the ones almost not reachable), we need to integrate the learning for reference vector into the main algorithm, and much technical effort is made to enable the analysis for the integrated algorithm.

\paragraph{Description of the Algorithm.} \UCBADV is described in Algorithm~\ref{alg1}, where $c_1$, $c_2$, and $c_3$ are large enough positive universal constants so that concentration inequalities may be applied in the analysis. Besides the standard quantities such as $Q_h(s, a)$, $V_h(s)$, and the reference value function $V^{\mathrm{ref}}_h$, the algorithm keeps seven types of accumulators to facilitate the update to the $Q$- and value functions: accumulators $N_h(s, a)$ and $\check{N}_h(s, a)$ are used to keep the total visit number and the number of visits only counting the current stage to $(s, a, h)$, respectively. Three types of \emph{intra-stage} accumulators are used for the samples in the latest stage; they are reset at the beginning of each stage and updated at every time step as follows (note that short-hands are defined for succinct presentation of the $Q$-function update rule in \eqref{equpdate}):
\begin{align}
    &\check{\mu}:= \check{\mu}_{h}(s_{h},a_{h})\stackrel{+}{\leftarrow} V_{h+1}(s_{h+1})-V^{\mathrm{ref}}_{h+1}(s_{h+1});& 
	       &\check{\upsilon} :=\check{\upsilon}_{h}(s_{h},a_{h})\stackrel{+}{\leftarrow}V_{h+1}(s_{h+1}); \label{eq:alg-acc-1}\\
	       &\check{\sigma} := \check{\sigma}_{h}(s_{h},a_{h})\stackrel{+}{\leftarrow} ( V_{h+1}(s_{h+1})-V^{\mathrm{ref}}_{h+1}(s_{h+1}))^2 . \label{eq:alg-acc-2}
\end{align}
Finally, the following two types of \emph{global} accumulators are used for the samples in all stages,
\begin{align}
&\mu^{\mathrm{ref}} := \mu^{\mathrm{ref}}_{h}(s_{h},a_{h})\stackrel{+}{\leftarrow}V_{h+1}^{\mathrm{ref}}(s_{h+1}); & 
	       &\sigma^{\mathrm{ref}} := \sigma^{\mathrm{ref}}_{h}(s_{h},a_{h})\stackrel{+}{\leftarrow} (V_{h+1}^{\mathrm{ref}}(s_{h+1}))^2. \label{eq:alg-acc-3}
\end{align}
  All accumulators are initialized to $0$ at the beginning of the algorithm.

 The algorithm sets $\iota \leftarrow \log(\frac{2}{p})$ (where $p$ is the parameter for the failure probability) and $\beta \leftarrow \frac{1}{\sqrt{H}}$. We also set $N_{0}:=\frac{c_4 SAH^5\iota}{\beta^2}$ for a large enough universal constant $c_4 > 0$, denoting the number of visits needed for each state to learn a $\beta$-accurate reference value.

By the definition of the accumulators, the first two expressions in $\min\{\cdot\}$ in \eqref{equpdate} respectively correspond to update rules \eqref{eqex0} and \eqref{eqex}, where $b$ and $\bar{b}$ are the respective exploration bonuses. The bonuses are set in a way that both expressions can be shown to upper bound $Q^*$ in the desired event. The update \eqref{equpdate} also makes sure that the learned $Q$-function is non-increasing as the algorithm proceeds.

\begin{algorithm}[tb]
	\caption{\UCBADV}
	\begin{algorithmic}\label{alg1}
		\STATE{\textbf{Initialize:} set all accumulators to $0$; for all $(s,a,h)\in \mathcal{S}\times \mathcal{A}\times [H]$, set $V_{h}(s)\leftarrow H-h+1$; $Q_{h}(s,a)\leftarrow H-h+1$; $V^{\mathrm{ref}}_{h}(s,a) \leftarrow H$;
	   }
	   \FOR{episodes $k \leftarrow 1,2,\dots,K$}
	   \STATE{observe $s_{1}$;}
	   \FOR{$h \leftarrow 1,2,\dots,H$}
	   \STATE{Take action $a_{h}\leftarrow  \arg\max_{a}Q_{h}(s_{h},a)$, and observe $s_{h+1}$.}
	   \STATE{
	   Update the accumulators via $n:=N_{h}(s_{h},a_{h})\stackrel{+}{\leftarrow}1$, $\check{n} := \check{N}_{h}(s_{h},a_{h})\stackrel{+}{\leftarrow}1$, and \eqref{eq:alg-acc-1}, \eqref{eq:alg-acc-2}, \eqref{eq:alg-acc-3}.}
	   \IF{$n\in \mathcal{L} $ \COMMENT{{\it Reaching the end of the stage and update triggered}} }  
	   \STATE{\{{\it Set the exploration bonuses, update the $Q$-function and the value function}\}}
	   \vspace{-3ex}
	   \STATE{\begin{align}
	   \hspace{-1ex} & \resizebox{.84\hsize}{!}{$b  \leftarrow  c_{1}\sqrt{\frac{\sigma^{\mathrm{ref}}/n- (\mu^{\mathrm{ref}}/n)^2 }{n}\iota}   +c_{2} \sqrt{\frac{   \check{\sigma}/\check{n} -(\check{\mu}/\check{n})^2  }{\check{n}}\iota  }+c_{3}(\frac{H\iota}{n}+\frac{H\iota}{\check{n}}+\frac{H\iota^{\frac{3}{4}}}{n^{\frac{3}{4}}}+\frac{H\iota^{\frac{3}{4}}}{\check{n}^{\frac{3}{4}}}  );$} \\
	   \hspace{-1ex} & \resizebox{.145\hsize}{!}{$\overline{b} \leftarrow 2\sqrt{\frac{ H^2 }{\check{n}}  \iota};$} \\
	   \label{equpdate}
	   \hspace{-1ex} & \resizebox{.84\hsize}{!}{$\displaystyle{Q_{h}(s_{h},a_{h}) \leftarrow \min\{  r_{h}(s_{h},a_{h})+\frac{\check{\upsilon}}{\check{n}} +\overline{b} ,\, r_{h}(s_{h},a_{h}) + \frac{\mu^{\mathrm{ref}}}{n}+\frac{\check{\mu}}{\check{n}}+b , \, Q_{h}(s_{h},a_{h})\};}$} \\
	   \hspace{-1ex} &V_{h}(s_{h})\leftarrow \max_{a}Q_{h}(s_{h},a); \label{equupdateV}
	   \end{align}}
	   \vspace{-2ex}
	   \STATE{$\check{N}_{h}(s_{h},a_{h}), \check{\mu}_{h}(s_{h},a_{h}), \check{\upsilon}_{h}(s_{h},a_{h} ), 
	   \check{\sigma}_{h}(s_{h},a_{h}) \leftarrow 0$; \{{\it Reset intra-stage accumulators}\}}
	   \ENDIF
	   \STATE{ \textbf{if} $\sum_{a}N_{h}(s_{h},a)= N_{0}$ \textbf{then} $V^{\mathrm{ref}}_{h}(s_{h})\leftarrow V_{h}(s_{h})$; } \COMMENT{{\it Learn the reference value function}}
	   \ENDFOR

	   \ENDFOR
	\end{algorithmic}
\end{algorithm}

\section{The Analysis (Proof of Theorem 1)}

Let  $N^{k}_{h}(s,a)$, $\check{N}^{k}_{h}(s,a)$, $Q^k_{h}(s,a)$, $V^{k}_h(s)$ and $V^{\mathrm{ref},k}_{h}(s)$ respectively denote the values of $N_{h}(s,a)$, $\check{N}_{h}(s,a)$, $Q_h(s,a)$, $V_h(s)$ and $V^{\mathrm{ref}}_h(s)$ at the beginning  of $k$-th episode. In particular, $N^{K+1}_{h}(s,a)$ denotes the  number of visits of $(s,a,h)$ after all $K$ episodes are done.  

To facilitate the proof, we need a few more notations. For each $k$ and $h$, let $n_h^k$ be the total number of visits to $(s_h^k, a_h^k, h)$ prior to the current stage with respect to the same triple. Let $\check{n}_h^{k}$ be the number of visits to $(s_h^k, a_h^k, h)$ during the stage immediately before the current stage. 
We let $l_{h, i}^k$ denote the index of the $i$-th episode among the $n_h^k$ episodes defined above. Also let $\check{l}_{h, i}^{k}$ be the index of the $i$-th episode among the $\check{n}_h^{k}$ episodes defined above. When $h$ and $k$ are clear from the context, we omit the two letters and use $l_{i}$ and $\check{l}_{i}$ for short.
 We  use $\mu^{\mathrm{ref},k}_h$, $\check{\mu}_h^k$, $\check{\nu}_h^k$, $\sigma^{\mathrm{ref},k}_h$, $\check{\sigma^{k}_h}$, $b_{h}^k$ and $\overline{b}_h^k$ to denote respectively the values of  $\mu^{\mathrm{ref}}$, $\check{\mu}$, $\check{\upsilon}$, $\sigma^{\mathrm{ref}}$, $\check{\sigma}$, $b$ and $\overline{b}$ in the computation of $Q_h^k(s_h^k,a_h^k)$ in \eqref{equpdate}.

Recall that the value function $Q_{h}(s,a)$ is non-increasing as the algorithm proceeds. On the other hand, we claim in the following proposition that $Q_{h}(s,a)$ upper bounds $Q^*_{h}(s,a)$ with high probability.

\begin{proposition}\label{pro1} Let $p\in (0,1)$.
With probability at least $(1-4T(H^2T^3+3))p$, it holds that 
$Q^*_{h}(s,a)\leq Q^{k+1}_{h}(s,a) \leq Q^{k}_{h}(s,a)$ 
for any $s,a,h,k$.
\end{proposition}
The proof of Proposition~\ref{pro1} involves some careful application of the concentration inequalities for martingales and is deferred to Appendix~\ref{app:B}.

\subsection{Learning the Reference Value Function}
As mentioned before, we hope to get an accurate estimate of $V^*$ as the reference value function. Similar to the  proof of Lemma 2 in \citep{dong2019q}, we show in the following lemma (the proof of which deferred to Appendix~\ref{app:B}) that  we can learn a good reference value for each state with bounded sample complexity. Also note that while it is possible to improve the upper bound in Lemma~\ref{lemma1} via more refined analysis, the current form is sufficient to prove our main theorem.

\begin{lemma}\label{lemma1}
Conditioned on the successful events of Proposition~\ref{pro1},	 for any $\epsilon\in (0,H]$, with probability $(1-Tp)$ it holds that  for any $h\in [H]$, 
	$\sum_{k=1}^{K}\mathbb{I} \left[ V^{k}_{h}(s_{h}^k) - V^*_h(s_{h}^k)\geq \epsilon  \right]\leq O({SAH^{5}\iota}/{\epsilon^2})$ .
\end{lemma}

 By Lemma \ref{lemma1} with $\epsilon$ set to $\beta$,  the fact that $V^k$ is non-increasing in $k$ and the definition of $N_0$, we have the following corollary.
\begin{corollary}\label{coro1}
Conditioned on the successful events of Proposition~\ref{pro1} and Lemma~\ref{lemma1}, for every state $s$ we have that 
 $n_{h}^{k}(s)\geq N_{0}\Longrightarrow V^*_{h}(s)\leq V^{\mathrm{ref},k}_{h}(s)\leq V^*_{h}(s)+\beta$.
 \end{corollary}

\subsection{Regret Analysis with Reference-Advantage Decomposition
}

We now prove Theorem~\ref{thm1}. We start by replacing $p$ by $p/\mathrm{poly}(H,T)$ so that we only need to show the desired regret bound with probability $(1 - \mathrm{poly}(H,T)\cdot p)$.
The proof in this subsection will also be conditioned on the successful events in Proposition~\ref{pro1} and Lemma~\ref{lemma1}, so that the regret can be expressed as
 	\begin{equation}
	\mathrm{Regret}(T) = \sum_{k=1}^{K}\big(V^*_{1}(s^k_{1})-V^{\pi_{k}}_{1}(s^k_{1})\big)\leq \sum_{k=1}^{K} \big(V^{k}_{1}(s^k_{1})-V^{\pi_{k}}_{1}(s^k_{1})\big) .
	\end{equation}
	Define $\delta_h^k := V_h^k(s_h^k)-V^*_{h}(s_h^k)$ and $\zeta^k_{h} := V^{k}_{h}(s_{h}^k)-V^{\pi_{k}}_{h}(s_{h}^k)$. 
	Note that when $N^k_{h}(s^k_h,a^k_h)\in \mathcal{L}$,  we have that $n_h^k =N^k_{h}(s^k_h,a^k_h) $ and $\check{n}_h^{k}= \check{N}^{k}_{h}(s^k_h,a^k_h)  $. Following the update rules \eqref{equpdate} and \eqref{equupdateV}, we have that\footnote{Here we define $0/0$ to be $0$ so that forms such as $\frac{1}{n_h^k} \sum_{i=1}^{n_h^k} X_i$ are treated as $0$ if $n_h^k = 0$.}
	 
\begin{align}
&V_{h}^k(s_h^k) \leq  \mathbb{I}\left[n_h^k = 0\right]  H+   r_{h}(s^k_h, a^k_h)   +\frac{\mu_h^{\mathrm{ref},k}}{n_h^k} +\frac{\check{\mu}_h^k}{\check{n}^k_h}+ b_h^k \nonumber \\
&=  \mathbb{I}\left[n_{h}^k=0\right]  H+ r_h(s_h^k,a_h^k)+ \frac{1}{n^{k}_{h}}\sum_{i=1}^{n^{k}_{h}}V^{\mathrm{ref},l_{i}}_{h+1}(s_{h+1}^{l_{i}})  +\frac{1}{\check{n}^{k}_{h}}\sum_{i=1}^{\check{n}^{k}_{h}}\big(V^{\check{l}_{i}}_{h+1}(s_{\check{l}_{i},h+1})  - V_{h+1}^{\mathrm{ref},\check{l}_{i}  }(s_{ \check{l}_{i},h+1 })\big) + b_h^k. \nonumber\end{align}

Together with the Bellman equation $V^{\pi_{k}}_{h}(s_h^k) = r_{h}(s^k_h,a^k_h) + P_{s^k_h,a^k_h,h}V^{\pi_{k}}_{h+1}$, we have that
\begin{align}
&\zeta^k_{h}= V^{k}_{h}(s^k_{h})-V^{\pi_{k}}_{h}(s^k_{h})\nonumber
\\ & \leq  \mathbb{I}\left[n_{h}^k=0\right]H+  \frac{1}{n^{k}_{h}}\sum_{i=1}^{n^{k}_{h}}V^{\mathrm{ref},l_{i}}_{h+1}(s_{h+1}^{l_{i}})+ \frac{1}{\check{n}^{k}_{h}}\sum_{i=1}^{\check{n}^{k}_{h}}\big(V^{\check{l}_{i}}_{h+1}(s_{\check{l}_{i},h+1})  - V_{h+1}^{\mathrm{ref},\check{l}_{i}  }(s_{ \check{l}_{i},h+1 })\big)  \nonumber \\ 
& \qquad + b_h^k - P_{s^k_{h},a^k_{h},h}V^{\pi_{k}}_{h+1}\nonumber
\\ & \leq  \mathbb{I}\left[n_{h}^k=0\right]H+ \frac{1}{n^{k}_{h}}\sum_{i=1}^{n^{k}_{h}}P_{s^k_{h},a^k_{h},h}V_{h+1}^{\mathrm{ref},l_{i}} +\frac{1}{\check{n}^k_{h}}\sum_{i=1}^{\check{n}^k_{h}} P_{s^k_{h},a^k_{h},h}(  V^{\check{l}_{i}}_{h+1}-    V_{h+1}^{\mathrm{ref},\check{l}_{i}  } )\nonumber 
\\
&    \qquad + 2b_h^k -P_{s^{k}_{h},a^k_{h},h}V_{h+1}^{\pi_{k}} \label{eqbf1}
\\ & = \mathbb{I}\left[n_{h}^k=0\right]H+ P_{s^k_{h},a^k_{h},h}\big( \frac{1}{n^{k}_{h}}\sum_{i=1}^{n^{k}_{h}}V^{\mathrm{ref},l_{i}}_{h+1}  - \frac{1}{\check{n}^k_{h}}\sum_{i=1}^{\check{n}^k_{h}} V_{h+1}^{\mathrm{ref},\check{l}_{i}  } \big) \nonumber \\
& \qquad +P_{s^k_{h},a^k_{h},h}\big(\frac{1}{\check{n}^k_{h}}\sum_{i=1}^{\check{n}^k_{h}}(V^{\check{l}_{i}}_{h+1}- V^*_{h+1}  ) \big) + P_{s^k_{h},a^k_{h},h}(V^{*}_{h+1}-V^{\pi_{k}}_{h+1})+ 2b_h^k \nonumber
\\ & \leq \mathbb{I}\left[n_{h}^k=0\right]H+ \frac{1}{\check{n}_{h}^{k}}\sum_{i=1}^{\check{n}_{h}^{k}}\delta_{h+1}^{\check{l}_{i}} -\delta_{h+1}^{k}+ \zeta_{h+1}^{k} +\underbrace{\psi_{h+1}^{k}+ \xi_{h+1}^{k}  +\phi_{h+1}^{k} + 2b_h^k}_{\Lambda_{h+1}^k}, \label{eqstar2}
\end{align}
where letting $V^{\mathrm{REF}}$ be the final reference vector (i.e.,  $V^{\mathrm{REF}} := V^{\mathrm{ref},K+1}$), and $\textbf{1}_j$ be the $j$-th canonical basis vector (i.e., $(0, \dots, 0, 1, 0, \dots, 0)$ where the only $1$ is located at the $j$-th entry), we define
\begin{align*}
 \psi_{h+1}^{k}& := \frac{1}{n_{h}^k}\sum_{i=1}^{n_{h}^k} P_{s_{h}^k,a_{h}^k,h}(V_{h+1}^{\mathrm{ref},l_{i}}-V_{h+1}^{\mathrm{REF}}), & \xi_{h+1}^{k} & := \frac{1}{\check{n}_{h}^{k}}\sum_{i=1}^{\check{n}_{h}^{k}}(P_{s_{h}^k,a_{h}^k,h}-\textbf{1}_{s_{h+1}^{\check{l}_{i}}})(V^{\check{l}_{i}}_{h+1}- V^*_{h+1}),\\
 \phi_{h+1}^k & := (P_{s_{h}^k,a_{h}^k,h}-\textbf{1}_{s_{h+1}^k}) (V^{*}_{h+1}-V^{\pi_{k}}_{h+1}).
\end{align*}

Here at Inequality \eqref{eqbf1} is implied by the successful event of martingale concentration (which is implied by the successful event in the proof of Proposition~\ref{pro1}, in particular, Inequality \eqref{eq:UCB-1}).  Inequality \eqref{eqstar2} holds by the fact that $V^{\mathrm{ref},k}_{h+1}\geq V^{ \mathrm{REF}}_{h+1}$ for any $k,h$. Now we turn to bound $\sum_{k=1}^{K}\zeta_{h}^{k}$. Note that 
\begin{align}\label{eqsec4.1}
&\sum_{k=1}^K \zeta_{h}^k  \leq \sum_{k=1}^{K}\mathbb{I}\left[ n_h^k=0\right]H + \sum_{k=1}^K (\frac{1}{\check{n}_{h}^{k}}\sum_{i=1}^{\check{n}_{h}^{k}}\delta_{h+1}^{\check{l}^{k}_{h,i}}) +\sum_{k=1}^K (\zeta_{h+1}^k +\Lambda_{h+1}^k -\delta_{h+1}^k) .
\end{align}
The first term in the \textbf{RHS} of $(\ref{eqsec4.1})$ is  bounded by $\sum_{k=1}^K \mathbb{I}[n_h^k=0]\leq SAH$ because $n_h^k\geq H$ when $N_h^k(s_h^k,a_h^k)\geq H$. We rewrite the second term as
\begin{align}
\sum_{k=1}^K (\frac{1}{\check{n}_{h}^{k}}\sum_{i=1}^{\check{n}_{h}^{k}}\delta_{h+1}^{\check{l}^{k}_{h,i}})  = \sum_{k=1}^{K}\frac{1}{\check{n}^k_h} \sum_{j=1}^K \delta_{h+1}^j\sum_{i=1}^{\check{n}^k_h} \mathbb{I}[j = \check{l}^{k}_{h,i}] = \sum_{j=1}^{K}\delta^j_{h+1} \sum_{k=1}^K \frac{1}{\check{n}^k_h} \sum_{i=1}^{\check{n}^k_h} \mathbb{I}[j = \check{l}^{k}_{h,i}] . \label{eq-analysis-rearrangejk}
\end{align} 
Let $j\geq 1$ be a fixed episode. Note that  $\sum_{i=1}^{\check{n}_h^{k} }\mathbb{I}[j = \check{l}^{k}_{h,i}]=1$ if and only if $(s_h^j,a_h^j) = (s_h^k,a_h^k)$, and $(j,h)$ falls in the previous stage that $(k,h)$ falls in. As a result, every $k$ such that $\sum_{i=1}^{\check{n}_h^{k} }\mathbb{I}[j = \check{l}^{k}_{h,i}]=1$ has the same $\check{n}^k_{h}$ which we denote by $Z_{j}$, and the set $\{k: \sum_{i=1}^{\check{n}_h^{k} }\mathbb{I}[j = \check{l}^{k}_{h,i}]=1  \} $ has at most $(1+\frac{1}{H})Z_{j}$ elements. Therefore, for every $j$ we have that
\begin{equation}\label{eq-analysis-rearrangejk-1}
\sum_{k=1}^K \frac{1}{\check{n}^k_h} \sum_{i=1}^{\check{n}^k_h} \mathbb{I}[j = \check{l}^{k}_{h,i}] \leq 1 + \frac{1}{H}.
\end{equation}

Because $\delta_{h+1}^k\leq \zeta_{h+1}^k$, combining \eqref{eqsec4.1}, \eqref{eq-analysis-rearrangejk}, and \eqref{eq-analysis-rearrangejk-1}, we have that
\begin{align}
\sum_{k=1}^{K}\zeta_{h}^{k} &\leq SAH^2+ (1+\frac{1}{H})\sum_{k=1}^{K}\delta_{h+1}^{k}-\sum_{k=1}^{K}\delta_{h+1}^{k}+\sum_{k=1}^{K}\zeta_{h+1}^{k}+\sum_{k=1}^{K}\Lambda_{h+1}^{k} \nonumber \\ 
& \leq SAH^2 +(1+\frac{1}{H})\sum_{k=1}^{K}\zeta_{h+1}^{k}+\sum_{k=1}^{K}\Lambda_{h+1}^k.\label{eqpf2}
\end{align}
Iterating the derivation above for $h=1,2,\cdots,H$ and  we have that
\begin{equation}\label{eqpf3}
\sum_{k=1}^{K}\zeta_{1}^{k}\leq O\Big(SAH^3+\sum_{h=1}^{H}\sum_{k=1}^{K}(1+\frac{1}{H})^{h-1}\Lambda_{h+1}^{k}\Big).
\end{equation}

We bound $\sum_{h=1}^{H}\sum_{k=1}^K  (1+\frac{1}{H})^{h-1}\Lambda_{h+1}^k$ in the lemma below. The detailed proof is deferred to Appendix~\ref{app:B} due to space constraints.
\begin{lemma}\label{bdlm} With probability at least $(1-O(H^2T^4p))$, it holds that 
\begin{equation}
\sum_{h=1}^{H}\sum_{k=1}^{K} (1+\frac1H)^{h-1} \Lambda_{h+1}^k = O\Big(\sqrt{SAH^2T\iota} + H\sqrt{T\iota} \log(T)  + S^2 A^{\frac32} H^8 \iota T^{\frac14} \Big).
\end{equation}
	
\end{lemma}

Combining Proposition \ref{pro1}, Lemma \ref{lemma1}, \eqref{eqpf3} and Lemma \ref{bdlm}, we conclude that with probability at least $(1-  O(H^2T^4p))$,
\begin{align*}
\mathrm{Regret}(T) = \sum_{k=1}^{K}\zeta_{1}^k =O\Big(\sqrt{SAH^2T\iota} + H\sqrt{T\iota} \log(T)  + S^2 A^{\frac32} H^8 \iota T^{\frac14} \Big).
\end{align*}


\bibliography{reference}
\bibliographystyle{plainnat}

\newpage 
\appendix
\appendixpage
\renewcommand{\appendixname}{Appendix~\Alph{section}}
\setlength{\parindent}{0pt}
\setlength{\parskip}{0.2\baselineskip}

\section{Basic Lemmas}
\begin{lemma}[Azuma-Hoeffding Inequality]
    Suppose $\{ X_{k}\}_{k=0,1,2,\dots}$ is a martingale and $|X_k-X_{k-1}|\leq c_{k}$ almost surely. Then for all positive integers $N$ and all positive reals $\epsilon$, it holds that
    \begin{align}
      \mathbb{P}\left[  |X_N-X_0| \geq \epsilon \right]\leq 2\exp\left(\frac{-\epsilon^2}{2\sum_{k=1}^N c_k^2}\right).        \nonumber
    \end{align}
\end{lemma}

\begin{lemma}[Freedman's Inequality, Theorem 1.6 of \citep{freedman1975tail}]\label{freedman}
	Let $(M_{n})_{n\geq 0}$ be a  martingale such that $M_{0}=0$ and $|M_{n}-M_{n-1}|\leq c$. Let $\mathrm{Var}_{n}=\sum_{k=1}^{n}\mathbb{E}[(M_{k}-M_{k-1})^{2}|\mathcal{F}_{k-1}]$ for $n\geq 0$,
	where $\mathcal{F}_{k}=\sigma(M_0,M_{1},M_{2},...,M_{k})$. Then, for any positive $x$ and for any positive $y$,
	\begin{equation}\label{Bernstein2}
	\mathbb{P}\left[ \exists n:  M_{n}\geq x ~\text{and}~\mathrm{Var}_{n}\leq y \right]  \leq \exp\left(-\frac{x^{2}}{2(y+cx)} \right).
	\end{equation}
\end{lemma}

\begin{lemma}\label{self-norm}
Let $(M_{n})_{n\geq 0}$ be a  martingale  such that $M_0 = 0$  and $|M_{n}-M_{n-1}|\leq c$ for some $c>0$ and any $n\geq 1$. Let $\mathrm{Var}_{n}=\sum_{k=1}^{n}\mathbb{E}[(M_{k}-M_{k-1})^{2}|\mathcal{F}_{k-1}]$ for $n\geq 0$,
where $\mathcal{F}_{k}=\sigma(M_{1},M_{2},...,M_{k})$. Then for any positive integer $n$, and any $\epsilon,p>0$, we have that
\begin{equation}\label{self-bernstein}
\mathbb{P}\left[|M_{n}|\geq   2\sqrt{\mathrm{Var}_{n}\log(\frac{1}{p})}+2\sqrt{\epsilon\log(\frac{1}{p} )} +2c\log(\frac{1}{p}) \right] \leq \left(\frac{2nc^2}{\epsilon}+2\right)p.
\end{equation}

\end{lemma}

\begin{proof}
For any fixed $n$, we apply Lemma \ref{freedman} with  $y = i\epsilon$ and $x= \pm(2\sqrt{y \log(\frac{1}{p})}+2c\log(\frac{1}{p}))$. For each $i=  1,2,\dots, \lceil \frac{nc^2}{\epsilon}\rceil$, we get that
\begin{align}
&  \mathbb{P}\left[|M_{n}|\geq 2\sqrt{(i-1)\epsilon\log(\frac{1}{p})}+ 2\sqrt{\epsilon\log(\frac{1}{p})}+2c\log(\frac{1}{p}),\mathrm{Var}_{n}\leq i\epsilon  \right] \nonumber
\\& \leq \mathbb{P}\left[ |M_{n}|\geq 2\sqrt{i\epsilon\log(\frac{1}{p})}+2c\log(\frac{1}{p}),\mathrm{Var}_{n}\leq i\epsilon  \right] \nonumber
\\& \leq 2p.
\end{align}
Then via a union bound, we have that
\begin{align}
&\mathbb{P}\left[|M_{n}|\geq  2\sqrt{\mathrm{Var}_{n}\log(\frac{1}{p})} +2\sqrt{\epsilon\log(\frac{1}{p})}+2c\log(\frac{1}{p}) \right] \nonumber
\\& \leq \sum_{i=1}^{\lceil \frac{nc^2}{\epsilon}\rceil}\mathbb{P}\left[ |M_{n}|\geq 2\sqrt{(i-1)\epsilon\log(\frac{1}{p})}+2\sqrt{\epsilon\log(\frac{1}{p})}+2c\log(\frac{1}{p}) ,(i-1)\epsilon\leq \mathrm{Var}_{n}\leq i\epsilon\right] \nonumber
\\& \leq  \sum_{i=1}^{\lceil \frac{nc^2}{\epsilon}\rceil}\mathbb{P}\left[|M_{n}|\geq 2\sqrt{(i-1)\epsilon\log(\frac{1}{p})}+ 2\sqrt{\epsilon\log(\frac{1}{p})}+2c\log(\frac{1}{p}),\mathrm{Var}_{n}\leq i\epsilon  \right] \nonumber
\\ & \leq \left(\frac{2nc^2}{\epsilon}+2\right)p .
\end{align} 
\end{proof}

\begin{lemma}\label{pro0} For any non-negative weights $\{ w_{h}(s,a)\}_{s\in \mathcal{S},a\in \mathcal{A},h\in [H]}$ and $\alpha\in (0,1)$, it holds that
	\begin{align}
&	\sum_{k=1}^{K}\sum_{h=1}^{H}  \frac{w_{h}(s_{h}^k,a_{h}^k)}{ (n_{h}^k)^{\alpha} } \leq \frac{2^{\alpha}}{1-\alpha}\sum_{s,a,h}w_{h}(s,a)(N^{K+1}_{h}(s,a))^{1-\alpha}, \label{eq25}
	\end{align}
	and 
	\begin{align}
	&\sum_{k=1}^{K}\sum_{h=1}^{H}  \frac{w_{h}(s_{h}^k,a_{h}^k)}{ (\check{n}_{h}^{k})^{\alpha} }\nonumber  \leq \frac{2^{2\alpha}H^\alpha}{1-\alpha}\sum_{s,a,h}w_{h}(s,a)(N^{K+1}_{h}(s,a))^{1-\alpha}.
	\end{align}
	In the case $\alpha = 1$, it holds that
	\begin{align}
&\sum_{k=1}^{K}\sum_{h=1}^{H}  \frac{w_{h}(s_{h}^k,a_{h}^k)}{ n_{h}^k }  \leq 2 \sum_{s,a,h}w_{h}(s,a)\log(N_h^{K+1}(s,a)),\label{eq26}
	\end{align}
and 
\begin{align}
&\sum_{k=1}^{K}\sum_{h=1}^{H}   \frac{w_{h}(s_{h}^k,a_{h}^k)}{ \check{n}_{h}^{k} } \nonumber
 \leq 4H \sum_{s,a,h}w_{h}(s,a)\log(N_h^{K+1}(s,a)) .
\end{align}

\end{lemma}
	\begin{proof}
		By the definition of $\mathcal{L}$, for any $h,k$ such that $n_{h}^{k}>0$, there exists $j$ such that $\check{n}_{h}^{k}=e_{j}$ and $n_{h}^k = \sum_{i=1}^{j}e_{i}$. Therefore, $\frac{1}{2H}n^k_{h} \leq  \check{n}_{h}^{k} \leq\frac{3}{H}n_{h}^k $. So it suffices to prove (\ref{eq25}) and (\ref{eq26}). By basic calculus,  for two positive numbers $x,y$ such that $y/2 \leq x\leq y$ and any $\alpha \in (0,1)$, we have that
		\begin{equation}\label{eqtool}
		y^{1-\alpha}-x^{1-\alpha}\geq (1-\alpha)(y-x)y^{-\alpha}\geq (1-\alpha)(y-x)2^{-\alpha} x^{-\alpha},
		\end{equation}
		and 
		\begin{equation}\label{eqtool1}
		\log(y)-\log(x)\geq \frac{y-x}{y}\geq 2 \frac{y-x}{x}.
		\end{equation}

	By applying (\ref{eqtool}) and $(\ref{eqtool1})$ with $y = \sum_{i=1}^{j+1}e_{i}$ and $x= \sum_{i=1}^{j}e_{i}$ for $j=1,2,...$ and taking sum, we have
		\begin{equation*}
		\begin{aligned}
		\sum_{k=1}^{K}\sum_{h=1}^{H}\frac{w_{h}(s_h^k,a_h^k)}{ ( n_{h}^k)^\alpha} & \leq \sum_{s,a,h}w_{h}(s,a)\sum_{j: \sum_{i=1}^j e_{i}\leq N^{K+1}_{h}(s,a)} \frac{\min\{ e_{j+1},N^{K+1}_h(s,a)-\sum_{i=1}^j e_i \}}{(\sum_{i=1}^j e_{j})^\alpha }
		\\ & \leq \frac{2^{\alpha}}{1-\alpha}\sum_{s,a,h}w_{h}(s,a)(N^{K+1}_{h}(s,a))^{1-\alpha}
		\end{aligned}
		\end{equation*}
		and
				\begin{equation*}
		\begin{aligned}
		\sum_{k=1}^{K}\sum_{h=1}^{H}\frac{w_{h}(s_h^k,a_h^k)}{  n_{h}^k} & \leq \sum_{s,a,h}w_{h}(s,a)\sum_{j: \sum_{i=1}^j e_{i}\leq N^{K+1}_{h}(s,a)} \frac{\min\{ e_{j+1},N^{K+1}_h(s,a)-\sum_{i=1}^j e_i \}}{(\sum_{i=1}^j e_{j}) }
		\\ & \leq 2\sum_{s,a,h}w_{h}(s,a)\log(N^{K+1}_h(s,a)).
		\end{aligned}
		\end{equation*}
		
	\end{proof}

\section{Missing Proofs in the Regret Analysis}\label{app:B}

\subsection{Proof of Proposition \ref{pro1}}
We prove $ Q^{*}_{h}(s,a)\leq Q^{k}_{h}(s,a)$ for all $k,h,s,a,$ by induction on $k$. Firstly, the conclusion holds when $k=1$. For $k\geq 2$, assume $Q^{*}_{h}(s,a)\leq Q^{u}_{h}(s,a)$ for any $h,s,a$ and $1\leq u\leq k$. Let $(s,a,h)$ be fixed.  If we do not update $Q_{h}(s,a)$ in the $k$-th episode, then $Q^{k+1}_{h}(s,a) = Q^{k}_{h}(s,a)\geq Q^*_{h}(s,a)$. Otherwise, we have
\begin{equation}
\begin{aligned}
Q_{h}^{k+1}(s,a) & = \min\Big\{    \underbrace{r_{h}(s,a)+\frac{\udl{\mu}^{\mathrm{ref}}}{\udl{n}}  +\frac{\udl{\check{\mu}}  }{\udl{\check{n}}}  +\udl{b} }_{(\romannumeral1)}       , \quad \underbrace{r_{h}(s,a)+\frac{\udl{\check{\upsilon}}   }{\udl{\check{n}} }+\udl{\overline{b}} }_{ (\romannumeral2) }       , \quad Q_{h}^{k}(s,a)\Big\},
\end{aligned}
\end{equation}
where $\udl{\mu}^{\mathrm{ref}}$, $\udl{\check{\mu}}$, $\udl{\sigma}^{\mathrm{ref}}$, $\udl{\check{\sigma}}$
, $\udl{n}$, $\udl{\check{n}}$, $\udl{b}$ and $\udl{\overline{b}}$ 
are given by  respectively the values of $\mu^{\mathrm{ref}}$, $\check{\mu}$, $\sigma^{\mathrm{ref}}$, $\sigma$, $n$, $\check{n}$, $b$ and $\overline{b}$ to compute $Q_h^{k+1}(s,a)$ in \eqref{equpdate}. We use $\udl{l}_i$ to denote the episode index of the $i$-th sample  and $\udl{\check{l}}_i$ to denote the episode index of the $i$-th sample of the last stage with respect to the triple $(s,a,h)$. 

Besides the last $Q_{h}^{k}(s,a)$ term, there are two non-trivial cases  to discuss (corresponding to ($\romannumeral1$) and ($\romannumeral2$)).

\medskip
\underline{\it For the first case,}  we have that

\begin{align}
&Q_{h}^{k+1}(s,a) = r_{h}(s,a) + \frac{\udl{\mu}^{\mathrm{ref}}}{\udl{n}}+\frac{\udl{\check{\mu}}  }{\udl{\check{n}}}  +\udl{b} \nonumber
\\& = r_{h}(s,a)+P_{s,a,h}\left(\frac{1}{\udl{n}}\sum_{i=1}^{\udl{n}}V^{\mathrm{ref},\udl{l}_{i}}_{h+1} \right) + P_{s,a,h}\left(\frac{1}{\udl{\check{n}}}\sum_{i=1}^{\udl{\check{n}}}(V^{\udl{\check{l}}_{i}}_{h+1}-V^{\mathrm{ref},\udl{\check{l}}_{i}}_{h+1})   \right) + \chi_{1}+\chi_{2}+\udl{b} \nonumber
\\ & \geq r_{h}(s,a)+  P_{s,a,h}\left(\frac{1}{\udl{\check{n}}}\sum_{i=1}^{\udl{\check{n}}}V_{h+1}^{\udl{\check{l}}_{i}}\right)    +  \chi_{1}+\chi_{2}+\udl{b}  \label{eq:pro4_1}
\\ & \geq r_{h}(s,a)+P_{s,a,h}V^*_{h+1}+\chi_{1}+\chi_{2}+\udl{b}  \label{eq:pro4_2}
\\ & = Q^*_{h}(s,a)+\chi_{1}+\chi_{2}+\underline{b} \nonumber 
\end{align}
where 
\begin{align}
&\chi_{1} := \frac{1}{\udl{n}} \sum_{i=1}^{\udl{n} }  \left(V^{\mathrm{ref},\udl{l}_{i}}_{h+1}(s_{h+1}^{\udl{l}_{i}})-P_{s,a,h}V_{h+1}^{\mathrm{ref},\udl{l}_{i}}  \right),\\&
W^{l}_{h+1}:= V^{l}_{h+1}-V^{\mathrm{ref},l}_{h+1},\quad  \forall l\geq 1\\ &
\chi_{2} := \frac{1}{\udl{\check{n}}} \sum_{i=1}^{\udl{\check{n}}  } \left( W^{\udl{\check{l}}_{i}}_{h+1}(s^{\udl{\check{l}}_{i}}_{h+1})-P_{s,a,h}W^{\udl{\check{l}}_{i}}_{h+1} \right).
\end{align}

Here, Inequality \eqref{eq:pro4_1} holds because $V^{\mathrm{ref},u}_{h+1}$ is non-increasing in $u$, Inequality \eqref{eq:pro4_2} is by the induction $V^{u}\geq V^*$ for any $1\leq u\leq k$.

Define $\mathbb{V}(x,y) := x^{\top}(y^2)-(x^{\top}y)^2$ for two vectors $x,y$ of the same dimension, where $y^2$ is obtained by squaring each entry of $y$.
By Lemma \ref{self-norm} with $\epsilon = \frac{1}{T^2}$, we have that with probability $(1-2(H^2T^3+1)p)$ it holds that

\begin{align}
& |\chi_{1}|\leq 2\sqrt{ \frac{ (\sum_{i=1}^{\udl{n}}\mathbb{V}(P_{s,a,h},V_{h+1}^{\mathrm{ref},\udl{l}_{i}})) \iota}{\udl{n}^2  }  \label{eq2} }+2\frac{\sqrt{\iota}}{T\udl{n}}+ \frac{2H\iota}{\udl{n} }
\\ & |\chi_{2}|\leq 2\sqrt{ \frac{ (\sum_{i=1}^{\udl{\check{n}}  }\mathbb{V}(P_{s,a,h},W_{h+1}^{\udl{\check{l}}_{i}})) \iota}{  \udl{\check{n}}   }  }+2\frac{\sqrt{\iota}}{T \udl{\check{n}} }+ \frac{2H\iota}{   \udl{\check{n}}  }.  \label{eq3}
\end{align}

We now bound $\sum_{i=1}^{\udl{\check{n}}}\mathbb{V}(P_{s,a,h},V_{h+1}^{\mathrm{ref},\udl{l}_{i}})$ in order to upper bound $|\chi_1|$. 
Define 
\[\udl{\nu}^{\mathrm{ref}} := \frac{\udl{\sigma}^{\mathrm{ref}}}{ \udl{n} }- \left(\frac{ \udl{\mu}^{\mathrm{ref}} }{\udl{n} }\right)^2.\]  
We claim that,
\begin{lemma}\label{bdnu}
With probability $(1-2p)$, it holds that 
	\begin{equation}\label{eq4}
	\sum_{i=1}^{\udl{n}   }\mathbb{V}(P_{s,a,h},V_{h+1}^{\mathrm{ref},\udl{l}_{i}})\leq \udl{n}\cdot \udl{\nu}^{\mathrm{ref}}+3H^2\sqrt{ \udl{n} \iota}.
	\end{equation}
	
	\end{lemma}
\begin{proof}
We have that 
\begin{align} \sum_{i=1}^{\udl{n} }&\mathbb{V}(P_{s,a,h},V_{h+1}^{\mathrm{ref},\udl{l}_{i}}): =\sum_{i=1}^{ \udl{n}} \left(P_{s,a,h}(V^{\mathrm{ref},\udl{l}_{i}}_{h+1})^2- (P_{s,a,h}V^{\mathrm{ref},\udl{l}_{i}}_{h+1})^2  \right) \nonumber
\\ & =\sum_{i=1}^{\udl{n}}(V_{h+1}^{\mathrm{ref},\udl{l}_{i}}(s^{\udl{l}_{i}}_{h+1 }))^2-\frac{1}{\udl{n}  } \left({\sum_{i=1}^{\udl{n} }V^{\mathrm{ref},\udl{l}_{i}}_{h+1}(s_{h+1}^{\udl{l}_{i}}) } \right)^2  +\chi_{3}+\chi_{4}+\chi_{5} \nonumber
\\ & = \udl{n} \cdot \udl{\nu}^{\mathrm{ref}}+\chi_{3}+\chi_{4}+\chi_{5}, \label{eqv1}
\end{align}
where 
\begin{align} 
& \chi_{3} :=\sum_{i=1}^{\udl{n}}\left(  P_{s,a,h}(V^{\mathrm{ref},\udl{l}_{i}}_{h+1})^2 - (V_{h+1}^{\mathrm{ref},\udl{l}_{i}}(s^{\udl{l}_{i}}_{h+1 }))^2  \right), \label{eqv2-a}
\\& \chi_{4} := \frac{1}{\udl{n}}\left({\sum_{i=1}^{\udl{n}}V^{\mathrm{ref},\udl{l}_{i}}_{h+1}(s_{h+1}^{\udl{l}_{i}}  ) }\right)^2 -\frac{1}{\udl{n} }\left({\sum_{i=1}^{\udl{n} } P_{s,a,h}V^{\mathrm{ref},\udl{l}_{i}}_{h+1} }\right)^2, \label{eqv2-b}
\\ & \chi_{5} :=   \frac{1}{\udl{n} }\left({\sum_{i=1}^{\udl{n} } P_{s,a,h}V^{\mathrm{ref},\udl{l}_{i}}_{h+1} }\right)^2  - \sum_{i=1}^{\udl{n}}\left(P_{s,a,h}V^{\mathrm{ref},\udl{l}_{i}}_{h+1}\right)^2 . \label{eqv2-c}
\end{align}
 By Azuma's inequality, we have $	|\chi_{3}|\leq  H^2\sqrt{2\udl{n}\iota}$
with probability at least $(1-p)$. We apply Azuma's inequality again to obtain that  with probability at least $(1-p)$, it holds that
\begin{align}
|\chi_{4}|&=\frac{1}{\udl{n}} \left|   \left(\sum_{i=1}^{\udl{n}}V^{\mathrm{ref},\udl{l}_{i}}_{h+1}(s_{h+1}^{\udl{l}_{i}})  \right)^2-  \left(\sum_{i=1}^{\udl{n}} P_{s,a,h}V^{\mathrm{ref},\udl{l}_{i}}_{h+1} \right)^2  \right| \nonumber
\\ & \leq 2H \cdot \left| \sum_{i=1}^{\udl{n}}V^{\mathrm{ref},\udl{l}_{i}}_{h+1}(s_{h+1}^{\udl{l}_{i}})-\sum_{i=1}^{\udl{n}} P_{s,a,h}V^{\mathrm{ref},\udl{l}_{i}}_{h+1}  \right| \nonumber
\\ & \leq 2H^2\sqrt{2\udl{n}\iota}.
\end{align}
On the other hand, we have that $\chi_{5}\leq 0$ by Cauchy-Schwartz inequality. The proof then is completed by \eqref{eqv1}.

\end{proof}
Combing (\ref{eq2}) with (\ref{eq4}) we have
\begin{equation}\label{eq5}
|\chi_{1}|\leq 2\sqrt{\frac{\udl{\nu}^{\mathrm{ref}}\iota}{\udl{n} }} +\frac{5H\iota^{\frac{3}{4}}}{(\udl{n} )^{\frac{3}{4}} } +\frac{2\sqrt{\iota}}{T \udl{n} }+ \frac{2H\iota}{ \udl{n} }.
\end{equation}

We now bound $\sum_{i=1}^{\udl{\check{n}}  }\mathbb{V}(P_{s,a,h},W_{h+1}^{\udl{\check{l}}_{i}})$ for $|\chi_2|$.
Define
\[ \udl{\check{\nu}}  := \frac{  \udl{\check{\sigma}}   }{ \udl{\check{n}}  }- \left(\frac{ \udl{\check{\mu}}  }{  \udl{\check{n}}  }\right)^2.\]
Similarly to Lemma~\ref{bdnu}, we have that
\begin{lemma}\label{bdchknu}
With probability $(1-2p)$, it holds that 
	\begin{equation}\label{eq4-a}
	\sum_{i=1}^{\udl{\check{n}} }\mathbb{V}(P_{s,a,h},W_{h+1}^{\udl{\check{l}}_{i}})\leq \udl{\check{n}} \cdot \udl{\check{\nu}} +3H^2\sqrt{\udl{\check{n}} \iota}.
	\end{equation}
\end{lemma}

Therefore, given \eqref{eq3}, it holds with probability $(1-2p)$ that
\begin{equation}\label{eq6}
|\chi_{2}|\leq  2\sqrt{\frac{\udl{\check{\nu}} \iota}{\udl{\check{n}}  }}+\frac{5H\iota^{\frac{3}{4}}}{(  \udl{\check{n}}  )^{\frac{3}{4}} } +\frac{2\sqrt{\iota}}{T\udl{\check{n}}  }+ \frac{2H\iota}{\udl{\check{n}}  }.
\end{equation}

Finally, combining \eqref{eq5}, \eqref{eq6}, and the definition of $\udl{b}$ with  $[c_1,c_2,c_3]  = [2,2,5]$, and collecting probabilities, we have that with probability at least $(1-2(H^2T^3+3))p$, it holds that 
\begin{align}\udl{b}\geq |\chi_{1}|+|\chi_{2}|, \label{eq:UCB-1}
\end{align}
which means that $Q^{k+1}_{h}(s,a)\geq Q^{*}_{h}(s,a)$.

\medskip
\underline{\it For the second case,} by Hoeffding's inequality, with probability $(1-p)$ it holds that 
\begin{align}
Q_{h}^{k+1}(s,a)& = r_{h}(s,a)+\frac{\udl{\check{\upsilon}}   }{\udl{\check{n}}  }+ \udl{\overline{b}} \nonumber 
\\& \geq r_{h}(s,a)+ \frac{1}{\udl{\check{n}}  }\sum_{i=1}^{\udl{\check{n}}  }V^*_{h+1}(s_{\udl{\check{l}}_{i},h+1} )  +2\sqrt{\frac{H^2}{\udl{\check{n}}  }\iota} \nonumber
\\ & \geq r_{h}(s,a)+P_{s,a,h}V^*_{h+1} \nonumber
\\ & = Q^*_{h}(s,a). 
\end{align}

Combining the two cases, and via a union bound over all time steps, we prove the proposition.

\subsection{Proof of Lemma \ref{lemma1}}

First, by Hoeffding's inequality, for every $k$ and $h$, we have that
\begin{equation}\label{eq:lemma1-1}
\mathbb{P}\left[\left|\frac{1}{\check{n}_h^k} \sum_{i=1}^{\check{n}_h^k} V^*_{h+1}(s_{h+1}^{\check{l}_i})-P_{s_h^k,a_h^k,h}V^*_{h+1}\right|\leq \overline{b}_h^k \right] > 1-p.
\end{equation}
Now the whole proof will be conditioned on that \eqref{eq:lemma1-1} holds for every $k$ and $h$, which happens with probability at least $(1 - Tp)$. For every $k$ and $h$, we let $\delta_{h}^k := V^{k}_{h}(s_{h}^k)-V^*_{h}(s_{h}^k)$ (which aligns with the definition for $\delta_h^k$ in the proof of Theorem~\ref{thm1}).

For any weight sequence $\{w^k\}_{k=1}^{K}$ such that $w^k \geq 0$, let $\|w\|_{\infty} = \max_{k=1}^K w^k$ and $\|w\|_{1} = \sum_{k=1}^K w^k$. We will prove that 
\begin{align}\label{eq:lemma1-2}
\sum_{k=1}^K w^k \delta_h^k \leq 240 H^\frac{5}{2}\sqrt{ \|w\|_{\infty} \cdot SA\|w\|_{1} \iota}+ 3SAH^3 \|w\|_{\infty}.
\end{align}
Once we have established \eqref{eq:lemma1-2}, we let $w^k =\mathbb{I}[\delta_h^k \geq \epsilon]$ and we have 
\[
\sum_{k=1}^K \mathbb{I}[\delta_h^k \geq \epsilon] \delta_h^k \leq 240 H^\frac{5}{2}\sqrt{ \|w\|_{\infty} \cdot SA\iota \sum_{k=1}^K \mathbb{I}[\delta_h^k \geq \epsilon]} + 3SAH^3 \|w\|_{\infty}.
\]
Note that $\|w\|_{\infty}$ is either $0$ or $1$. In either cases, we are able to derive that
\[
\sum_{k=1}^K \mathbb{I}[\delta_h^k\geq \epsilon]  \leq O(SAH^5 \iota /\epsilon^2),
\]
and concludes the proof of the lemma. Therefore, we only need to prove \eqref{eq:lemma1-2}, and the rest of the proof is devoted to establishing \eqref{eq:lemma1-2}.

By the update rule \eqref{equpdate} and \eqref{equupdateV}, that $V^k$ always upper bounds $V^*$ (conditioned on the successful event of Proposition~\ref{pro1}), and that we have conditioned on \eqref{eq:lemma1-1}, we have that 
\begin{align}
& \delta_h^k  =V^{k}_{h}(s_{h}^k) - V^*_{h}(s_{h}^k) \nonumber \\& \leq Q^{k}_{h}(s_{h}^k,a_{h}^k)- Q^*_{h}(s_{h}^k,a_{h}^k) \nonumber
\\ & \leq  \mathbb{I}[n_{h}^k=0]H + \big(\overline{b}_h^k+\frac{1}{\check{n}^k_{h}}\sum_{i=1}^{\check{n}^k_{h}}V_{h+1}^{\check{l}_{i}}(s^{\check{l}_{i}}_{h+1} )   - P_{s_{h}^k,a_{h}^k,h}V^*_{h+1}\big) \nonumber
\\ & \leq \mathbb{I}[n_{h}^k = 0]H+ \big( 2\overline{b}_h^k +  \frac{1}{\check{n}^k_{h}}\sum_{i=1}^{\check{n}^k_{h}}(  V_{h+1}^{\check{l}_{i}}(s^{\check{l}_{i}}_{h+1} )- V_{h+1}^*(s^{\check{l}_{i}}_{h+1}   ) ) \big) \nonumber
\\ &  =\mathbb{I}[n_{h}^k = 0]H+ \big( 2\overline{b}_h^k +  \frac{1}{\check{n}^k_{h}}\sum_{i=1}^{\check{n}^k_{h}}\delta_{h+1}^{\check{l}_i} \big). \label{eq2.0}
\end{align}

Using the similar trick we do for \eqref{eq-analysis-rearrangejk} and \eqref{eq-analysis-rearrangejk-1}, we have
\begin{align} \label{eq:lemma1-3}
  \sum_{k=1}^K  \frac{w^k}{\check{n}^k_{h}}\sum_{i=1}^{\check{n}^k_{h}} \delta_{h+1}^{\check{l}_i}& = \sum_{j=1}^K  \frac{w^j}{\check{n}^j_{h}}\sum_{i=1}^{\check{n}^j_{h}} \delta_{h+1}^{\check{l}_{h,i}^j} \nonumber\\
  & = \sum_{j=1}^K  \frac{w^j}{\check{n}^j_{h}}\sum_{k=1}^{K} \delta_{h+1}^{k} \sum_{i=1}^{\check{n}^j_{h}} \mathbb{I} [k = \check{l}_{h,i}^{j}] =   \sum_{k=1}^{K} \delta_{h+1}^k \sum_{j=1}^{K} \frac{w^j}{\check{n}_h^j}\sum_{i=1}^{\check{n}_h^j} \mathbb{I} [k=\check{l}^j_{h, i}],
\end{align}
where if we let 
\begin{align}\label{eq:lemma1-3a}
\tilde{w}^k = \sum_{j=1}^{K} \frac{w^j}{\check{n}_h^j}\sum_{i=1}^{\check{n}_h^j} \mathbb{I} [k=\check{l}^j_{h, i}],
\end{align}
we have that 
\begin{align}\label{eq:lemma1-3b}
    \|\tilde{w}\|_{\infty} =  \max_k \tilde{w}^k \leq (1 + \frac1H) \|w\|_{\infty}, \qquad \text{and} \qquad \|\tilde{w}\|_{1} =  \sum_k \tilde{w}^k  = \sum_{k} w^k = \|w\|_{1}.
\end{align}

Therefore, combining \eqref{eq2.0}, \eqref{eq:lemma1-3}, and \eqref{eq:lemma1-3a}, and plugging them into $\sum_{k}w^k \delta_{h}^k$, we have that 
\begin{align}
\sum_{k}w^k\delta_{h}^k &  \leq 2 \sum_{k}w^k\overline{b}_h^k +  \sum_{k}\tilde{w}^k \delta_{h+1}^k+H\sum_{k}w^k\mathbb{I}\left[n_{h}^k =0\right] \nonumber\\
&\leq 2 \sum_{k}w^k\overline{b}_h^k + \sum_{k}\tilde{w}^k \delta_{h+1}^k+SAH^2 \|w\|_{\infty},
\label{eq:lemma1-4}
\end{align}

We now bound the first term of \eqref{eq:lemma1-4}. Define $\mathfrak{w}(s,a,j) := \sum_{k=1}^K w^k\mathbb{I} [\check{n}_h^k = e_{j},(s_h^k,a_h^k)=(s,a)] $ and $\mathfrak{w}(s,a) := \sum_{j\geq 1}\mathfrak{w}(s,a,j)$. We have $\mathfrak{w}(s,a,j)\leq \|w\|_{\infty}(1+\frac{1}{H})e_{j}$ and $\sum_{s,a}\mathfrak{w}(s,a) = \sum_{k}w^k$. We then have
 \begin{align}
 \sum_{k}w^k \overline{b}_h^k & = \sum_{k} 2\sqrt{H^2\iota}w^k \sqrt{\frac{1}{\check{n}_h^k}} \nonumber \\
 &= 2\sqrt{H^2\iota} \sum_{s,a,j}\sqrt{\frac{1}{e_{j}}} \sum_{k=1}^K w^k\mathbb{I} [\check{n}_h^k = e_{j}, (s_h^k,a_h^k)=(s,a)] =2\sqrt{H^2\iota}\sum_{s,a}\sum_{j\geq 1}\mathfrak{w}(s,a,j)\sqrt{\frac{1}{e_{j}}}. \nonumber
 \end{align}
We fix $(s,a)$ and consider the sum $\sum_{j\geq 1}\mathfrak{w}(s,a,j)\sqrt{\frac{1}{e_{j}}}$. Notice that $\sqrt{1/e_j}$ is monotonically decreasing in $j$. Given that $\sum_{j\geq 0}\mathfrak{w}(s,a,j) = \mathfrak{w}(s,a)$ is fixed, by rearrangement inequality we have that 
\begin{align}
\sum_{j\geq 1}\mathfrak{w}(s,a,j)\sqrt{\frac{1}{e_{j}}} & \leq \sum_{j\geq 1}  \sqrt{\frac{1}{e_{j}}} \cdot \|w\|_{\infty} (1+\frac{1}{H}) e_j \cdot  \mathbb{I}\left[\sum_{i=1}^{j-1}\|w\|_{\infty} (1+\frac{1}{H})e_{i} \leq \mathfrak{w}(s,a) \right]   \nonumber\\
& = \|w\|_{\infty}(1+\frac{1}{H})\sum_{j}\sqrt{e_{j}} \cdot  \mathbb{I}\left[\sum_{i=1}^{j-1} \|w\|_{\infty} e_{i} \leq \mathfrak{w}(s,a) \right] \nonumber \\
&\leq 10(1+\frac{1}{H})\sqrt{\|w\|_{\infty} H \cdot \mathfrak{w}(s,a)} .\nonumber
\end{align}
Therefore, by Cauchy-Schwartz, we have that
\begin{align}
\sum_{k}w^k \overline{b}_h^k  \leq 2\sqrt{H^2\iota} \sum_{s,a}10(1+\frac{1}{H})\sqrt{\|w\|_{\infty} H}\sqrt{\mathfrak{w}(s,a)}\leq  20\sqrt{H^2\iota}(1+\frac{1}{H})\sqrt{\|w\|_{\infty} \cdot SAH\|w\|_{1}}. \label{eq:lemma1-5}
 \end{align}

Combining \eqref{eq:lemma1-4} and \eqref{eq:lemma1-5}, we have that
\begin{align} \label{eq:lemma1-6}
    \sum_{k}w^k\delta_{h}^k  \leq 80H\sqrt{\|w\|_{\infty} \cdot SAH \|w\|_{1}\iota} + SAH^2\|w\|_{\infty} + \sum_{k} \tilde{w}^k \delta_{h+1}^k . 
\end{align}
With \eqref{eq:lemma1-6} and \eqref{eq:lemma1-3b} in hand, applying induction on $h$ with the base case that $h = H$, one may deduce that 
\begin{align*} 
    \sum_{k}w^k\delta_{h}^k  &\leq (1+1/H)^H \cdot H \cdot \left(80 H\sqrt{ \|w\|_{\infty} \cdot SAH \|w\|_{1} \iota} + SAH^2 \|w\|_{\infty}\right)\\
    &\leq 240 H^2\sqrt{ \|w\|_{\infty} \cdot SAH \|w\|_{1} \iota} + 3 SAH^3 \|w\|_{\infty}.
\end{align*}

\subsection{Proof of Lemma \ref{bdlm}}

The entire proof is conditioned on the successful events of Proposition~\ref{pro1} and Lemma~\ref{lemma1} which happen with probability at least $(1 - 2T(H^2T^3+5)p)$. For convenience, we define $\lambda_{h}^{k}$ as $\lambda_{h}^k(s) = \mathbb{I} \left[n_{h}^{k}(s)<N_{0}\right]$ for all state $s$ and all $k$ and $h$.

By the definition of $\Lambda_{h+1}^{k}$, we have that  $\sum_{h=1}^{H}\sum_{k=1}^{K}(1+\frac{1}{H})^{h-1}\Lambda_{h+1}^{k}$ by the definition that
\begin{align}
\sum_{h=1}^{H}\sum_{k=1}^{K}(1+\frac{1}{H})^{h-1}\Lambda_{h+1}^{k} = & \sum_{h=1}^{H}\sum_{k=1}^{K} (1+\frac{1}{H})^{h-1} \psi_{h+1}^k  + \sum_{h=1}^{H}\sum_{k=1}^{K}(1+\frac{1}{H})^{h-1} \xi_{h+1}^k  \nonumber \\
&+ \sum_{h=1}^{H}\sum_{k=1}^{K}(1+\frac{1}{H})^{h-1}\phi_{h+1}^k + 2 \sum_{h=1}^{H}\sum_{k=1}^{K}(1+\frac{1}{H})^{h-1}  b_h^k. \label{eq:deflam}
\end{align}
We will bound the four terms separately.

\subsubsection{The $\psi_{h+1}^k$ Term}
\begin{lemma}\label{lemma:bd_psi}
With probability at least $(1-p)$, it holds that
\begin{align}
\sum_{h=1}^H\sum_{k = 1}^K (1+\frac{1}{H})^{h-1}\psi_{h+1}^k \leq O(\log(T)) \cdot (H^2SN_0+ H \sqrt{T \iota }).     \nonumber
\end{align}
\end{lemma}
\begin{proof}
  Because $\psi_{h+1}^{k} $ is always non-negative, we have that with probability $(1-p)$ it holds that
\begin{align}
&\sum_{k=1}^{K}\sum_{h=1}^{H}(1+\frac{1}{H})^{h-1}\psi_{h+1}^{k}  \nonumber
\\& \leq 3\sum_{k=1}^{K}\sum_{h=1}^{H}\psi_{h+1}^{k}  \nonumber
\\&= 3\sum_{k=1}^{K}\sum_{h=1}^{H}\frac{1}{n_{h}^k}\sum_{i=1}^{n_{h}^k} P_{s_{h}^k,a_{h}^k,h}(V_{h+1}^{\mathrm{ref},l_{i}}-V_{h+1}^{\mathrm{REF}  }) \nonumber
\\&\leq 3H \sum_{k=1}^{K}\sum_{h=1}^{H}\frac{1}{n_{h}^k}\sum_{i=1}^{n_{h}^k} P_{s_{h}^k,a_{h}^k,h} \lambda_{h+1}^{l_i} \nonumber
\\&\leq 3H \sum_{h=1}^{H}  \sum_{j=1}^{K} \sum_{k=1}^{K}  P_{s_{h}^k,a_{h}^k,h} \lambda_{h+1}^{j} \cdot  \frac{1}{n_{h}^k}\sum_{i=1}^{n_{h}^k}\mathbb{I}[l_{h,i}^k = j] \nonumber
\\&\leq 3H \sum_{h=1}^{H}  \sum_{j=1}^{K}   P_{s_{h}^j,a_{h}^j,h} \lambda_{h+1}^{j} \cdot  \sum_{k=1}^{K} \frac{1}{n_{h}^k}\sum_{i=1}^{n_{h}^k}\mathbb{I}[l_{h,i}^k = j]  \label{eq:logT-0} 
\\& \leq 6\big(\log(T)+1\big)\cdot H\sum_{h=1}^{H} \sum_{k=1}^{K}P_{s_{h}^{k},a_{h}^k,h}\lambda_{h+1}^k \label{eq:logT} 
\\& = 6\big(\log(T)+1\big) \cdot H\Big(\sum_{k=1}^{K}\sum_{h=1}^{H}\lambda_{h+1}^{k}(s_{h+1}^k)  +\sum_{k=1}^{K}\sum_{h=1}^{H} (P_{s_{h}^k,a_{h}^k,h} -\textbf{1}_{s_{h+1}^k} )\lambda_{h+1}^k    \Big) \nonumber
\\& \leq 6\big(\log(T)+1\big) \cdot H\Big(HSN_0  +\sum_{k=1}^{K}\sum_{h=1}^{H} (P_{s_{h}^k,a_{h}^k,h} -\textbf{1}_{s_{h+1}^k} )\lambda_{h+1}^k    \Big) \nonumber\\
& \leq 6\big(\log(T)+1\big) \cdot H\Big(HSN_0  +  2\sqrt{T \iota}   \Big)  . \label{eq:bound-psi-1} 
\end{align}

Here, Inequality \eqref{eq:logT-0} is because $\frac{1}{n_{h}^k}\sum_{i=1}^{n_{h}^k} \mathbb{I}[l_{h,i}^k = j] \neq 0$ only if $(s_h^k, a_h^k) = (s_h^j, a_h^j)$. Inequality  \eqref{eq:logT} is because 
\begin{align}
\sum_{k=1}^K \frac{1}{n_h^k}
\sum_{i=1}^{n_h^k}\mathbb{I}\left[l_{h,i}^k = j\right]\leq \sum_{z : j \leq \sum_{i=1}^{z-1} e_i  \leq T}\frac{e_z}{\sum_{i=1}^{z-1}e_i} \leq   2(\log(T)+1).\nonumber 
\end{align}
Inequality \eqref{eq:bound-psi-1} holds with probability $(1 - p)$ due to Azuma's inequality.

\end{proof}

\subsubsection{The $\xi_{h+1}^{k}$ Term} 
\begin{lemma}\label{lemma:bd_xi}
With probability at least $(1-(T+1) p)$, it holds that
\begin{align}
\sum_{k=1}^K \sum_{h=1}^H (1+\frac{1}{H})^{h-1}\xi_{h+1}^k \leq O(H\sqrt{SAT \iota}) .    \nonumber
\end{align}
    
\end{lemma}

\begin{proof}
We  have that
\begin{align*}
&\sum_{k=1}^{K}\sum_{h=1}^{H}(1+\frac{1}{H})^{h-1}\xi_{h+1}^k =\sum_{k=1}^{K}\sum_{h=1}^{H}(1+\frac{1}{H})^{h-1}\Big( \frac{1}{\check{n}_{h}^{k}}\sum_{i=1}^{\check{n}_{h}^{k}}(P_{s_{h}^k,a_{h}^k,h}-\textbf{1}_{s_{h+1}^{\check{l}_{i}}})(V^{\check{l}_{i}}_{h+1}- V^*_{h+1})\Big)
\\
& = \sum_{k=1}^{K}\sum_{h=1}^{H} \sum_{j=1}^{K} (1+\frac{1}{H})^{h-1}\Big( \frac{1}{\check{n}_{h}^{k}}\sum_{i=1}^{\check{n}_{h}^{k}}(P_{s_{h}^k,a_{h}^k,h}-\textbf{1}_{s_{h+1}^{j}})(V^{j}_{h+1}- V^*_{h+1}) \cdot \mathbb{I}[\check{l}_{h,i}^k = j]\Big).
\end{align*}
Note that in the expression above $\check{l}_{h,i}^k = j$ if and only if $(s_h^k, a_h^k) = (s_h^j, s_h^j)$. Therefore, we have 
\begin{align}
&\sum_{k=1}^{K}\sum_{h=1}^{H}(1+\frac{1}{H})^{h-1}\xi_{h+1}^k \nonumber\\
& = \sum_{k=1}^{K}\sum_{h=1}^{H} \sum_{j=1}^{K} (1+\frac{1}{H})^{h-1}\Big( \frac{1}{\check{n}_{h}^{k}}\sum_{i=1}^{\check{n}_{h}^{k}}(P_{s_{h}^j,a_{h}^j,h}-\textbf{1}_{s_{h+1}^{j}})(V^{j}_{h+1}- V^*_{h+1}) \cdot \mathbb{I}[\check{l}_{h,i}^k = j]\Big) \nonumber\\
& = \sum_{h=1}^{H} \sum_{j=1}^{K} (1+\frac{1}{H})^{h-1} (P_{s_{h}^j,a_{h}^j,h}-\textbf{1}_{s_{h+1}^{j}})(V^{j}_{h+1}- V^*_{h+1}) \cdot \sum_{k=1}^{K}\frac{1}{\check{n}_{h}^{k}}\sum_{i=1}^{\check{n}_{h}^{k}} \mathbb{I}[\check{l}_{h,i}^k = j] \nonumber \\
& = \sum_{k=1}^{K}\sum_{h=1}^{H}\theta_{h+1}^{k} (P_{s^k_{h},a_{h}^k,h }-\textbf{1}_{s_{h+1}^k})(V^{k}_{h+1}-V^*_{h+1}), \label{eq:bound-xi-1}
\end{align}
where we define $\theta_{h+1}^{j}:=(1+\frac{1}{H})^{h-1}\sum_{k=1}^{K} \big(\frac{1}{{\check{n}^k_{h}}}\sum_{i=1}^{\check{n}^k_{h}}\mathbb{I}[\check{l}^{k}_{h,i}=j] \big)$.

For $(j,h)\in [K]\times [H]$, let $x_{h}^j$ be the number of elements in current stage with respect to $(s_h^j,a_h^j,h)$ and
$\tilde{\theta}_{h+1}^{j}:=(1+\frac{1}{H})^{h-1}\frac{ \lfloor (1+\frac{1}{H})x_{h}^j\rfloor }{x_h^j}\leq 3$. 
Define $\mathcal{K} =\{(k,h): \theta_{h+1}^k = \tilde{\theta}_{h+1}^k \}$. Note that if $k$ is before the second  last stage (before the final episode $K$) of the triple $(s^k_h,a^k_h,h)$, then we have $\theta_{h+1}^k =\tilde{\theta}_{h+1}^{k} $ and $(k, h) \in \mathcal{K}$. Given that $(k, h) \in \mathcal{K}$, $s_{h+1}^k$ still follows the transition distribution $P_{s_h^k, a_h^k, h}$.  

Let $\mathcal{K}_h^{\bot}(s,a) = \{k: (s_h^k,a_h^k)=(s,a), k \text{ is in the second last stage of } (s,a,h)  \}$.
Note that for two different episodes $j,k$, if $(s_h^k,a_h^k) = (s_h^j,a_h^j)$ and $j,k$ are in the same stage of $(s_h^k,a_h^k,h)$, then 
$\theta_{h+1}^k = \theta_{h+1}^j$ and $\tilde{\theta}_{h+1}^k =\tilde{\theta}_{h+1}^j $.
Let $\theta_{h+1}(s,a)$ and $\tilde{\theta}_{h+1}(s,a)$ to denote $\theta^{k}_{h+1}$ and $\tilde{\theta}_{h+1}^k$ respectively for some $k\in \mathcal{K}_h^{\bot}(s,a).$

We rewrite as
\begin{align}
&\sum_{k=1}^{K}\sum_{h=1}^{H}(1+\frac{1}{H})^{h-1}\xi_{h+1}^k \nonumber\\
&= \sum_{(k, h)} \tilde{\theta}_{h+1}^{k} (P_{s^k_{h},a_{h}^k,h }-\textbf{1}_{s_{h+1}^k})(V^{k}_{h+1}-V^*_{h+1}) + \sum_{(k, h) \in \overline{\mathcal{K}}} (\theta_{h+1}^{k}-\tilde{\theta}_{h+1}^k) (P_{s^k_{h},a_{h}^k,h }-\textbf{1}_{s_{h+1}^k})(V^{k}_{h+1}-V^*_{h+1}) . \label{eq:bound-xi-2}
\end{align}

Because $\tilde{\theta}_{h+1}^k$ is independent from $s_{h+1}^k$, by Azuma's inequality, we have with probability $(1-p)$, it holds that

\begin{align}
    & \sum_{(k, h) } \tilde{\theta}_{h+1}^k (P_{s^k_{h},a_{h}^k,h }-\textbf{1}_{s_{h+1}^k})(V^{k}_{h+1}-V^*_{h+1})\leq 6\sqrt{TH^2\iota}.
    \label{eq:bound-xi-3}
\end{align}

For the second term in \eqref{eq:bound-xi-2}, we have that
\begin{align}
&\sum_{(k, h) \in \overline{\mathcal{K}}} (\theta_{h+1}^{k}-\tilde{\theta}_{h+1}^k) (P_{s^k_{h},a_{h}^k,h }-\textbf{1}_{s_{h+1}^k})(V^{k}_{h+1}-V^*_{h+1})  \nonumber \\
& = \sum_{s,a,h}\sum_{k : (k,h)\in \overline{\mathcal{K}}}\mathbb{I}[(s_h^k,a_h^k)=(s,a)](\theta_{h+1}^{k}-\tilde{\theta}_{h+1}^k) (P_{s^k_{h},a_{h}^k,h }-\textbf{1}_{s_{h+1}^k})(V^{k}_{h+1}-V^*_{h+1})  \nonumber \\
& =\sum_{s,a,h} (\theta_{h+1}(s,a)-\tilde{\theta}_{h+1}(s,a))\sum_{k\in \mathcal{K}_h^{\bot}(s,a)}   (P_{s^k_{h},a_{h}^k,h }-\textbf{1}_{s_{h+1}^k})(V^{k}_{h+1}-V^*_{h+1})  \nonumber      \\
& \leq \sum_{s,a,h}O(H)\sqrt{|\mathcal{K}_h^{\bot}(s,a)| \iota}\label{eq:bound-xi-4} \\
&   = \sum_{(s, a, h)} O(H) \cdot \sqrt{\check{N}^{K+1}_h(s, a) \iota}\nonumber \\
& \leq O(H) \cdot \sqrt{SAH \iota \sum_{(s, a, h)} \check{N}^{K+1}_h(s, a) } \label{eq:bound-xi-5}\\ 
& \leq O(H) \cdot \sqrt{SAH \iota \cdot (T/H)} . \label{eq:bound-xi-6}
\end{align}
Here, \eqref{eq:bound-xi-4} happens with probability $(1 - Tp)$ because of Azuma's inequality and a union bound over all times steps in $\overline{\mathcal{K}}$. \eqref{eq:bound-xi-5} is due to Cauchy-Schwartz, and \eqref{eq:bound-xi-6} is because the length of the last two stages for each $(s, a, h)$ triple is only $O(1/H)$ fraction of the total number of visits.

Combining \eqref{eq:bound-xi-2}, \eqref{eq:bound-xi-3}, \eqref{eq:bound-xi-6}, and collecting probabilities, we prove the desired result.

\end{proof}

\subsubsection{The $\phi_{h+1}^k$ Term} 
\begin{lemma}\label{lemma:bd_phi}
 With probability $(1-p)$, it holds that 
\begin{equation*}
\sum_{k=1}^{K}\sum_{h=1}^{H}(1+\frac{1}{H})^{h-1}\phi_{k+1}^{k} = \sum_{k=1}^{K}\sum_{h=1}^{H}(1+\frac{1}{H})^{h-1} (P_{s_{h}^k,a_{h}^k,h}-\textbf{1}_{s_{h+1}^k}) (V^{*}_{h+1}-V^{\pi_{k}}_{h+1})\leq O(\sqrt{H^2T\iota}).
\end{equation*}
\end{lemma}
\begin{proof}
The lemma follows easily from Azuma's inequality.
\end{proof}

\subsubsection{The $b_h^k$ Term}
\begin{lemma}\label{lemma:bd_b}
    With probability $(1-9p)$, it holds that
    \begin{align}
        \sum_{k=1}^K \sum_{h=1}^H (1+\frac{1}{H})^{h-1}b_h^k &\leq O\Big(\sqrt{SAH^2T\iota} +\sqrt{SAH^2\beta T\iota}+SAH^3\sqrt{SN_0 \iota}\log(T) \nonumber
        \\& \qquad \qquad         +\sqrt{SAH^3\beta^2  T\iota}+ (SA\iota)^{\frac{3}{4}}H^{\frac{5}{2}}T^{\frac{1}{4}}   \Big). \nonumber
    \end{align}
\end{lemma}
\begin{proof}
Define $\nu^{\mathrm{ref},k}_{h} = \frac{\sigma^{\mathrm{ref},k}}{ n_{h}^k}- (\frac{\mu^{\mathrm{ref},k}_{h} }{n_{h}^k})^2$ and $\check{\nu}_{h}^{k} = \frac{  \check{\sigma}_{h}^k  }{ \check{n}^k_{h} }- (\frac{ \check{\mu}^k_{h} }{  \check{n}^k_{h} })^2$.
Since $b_h^k$ is non-negative,  we have that
\begin{align}
&2\sum_{h=1}^{H}\sum_{k=1}^{K}(1+\frac{1}{H})^{h-1}b_h^k  \nonumber 
\\ & \leq  6\sum_{h=1}^{H}\sum_{k=1}^{K}\Big( c_{1}\sqrt{ \frac{ \nu^{\mathrm{ref},k}_{h}   }{n_{h}^k} \iota  }+c_{2}    \sqrt{ \frac{ \check{\nu}^{k}_{h}   }{\check{n}_{h}^{k}} \iota  }    +c_{3}\big(\frac{H\iota}{n_{h}^k} +\frac{H\iota}{\check{n}^k_{h}}  + \frac{H\iota^{\frac{3}{4}}}{(n_{h}^k)^{\frac{3}{4}}}+\frac{H\iota^{\frac{3}{4}}}{(\check{n}^k_{h})^{\frac{3}{4}}}   \big)     \Big) \label{eq:lower_order}
\\ & \leq O\Big(\sum_{h=1}^{H}\sum_{k=1}^{K} (\sqrt{ \frac{ \nu^{\mathrm{ref},k}_{h}   }{n_{h}^k} \iota  } + \sqrt{ \frac{ \check{\nu}^{k}_{h}   }{\check{n}_{h}^{k}} \iota  }    )    \Big) +O\Big(SAH^3\log(T)\iota   +(SA\iota)^{\frac{3}{4}}H^{\frac{5}{2}}T^{\frac{1}{4}} \Big). \label{eqbb}
\end{align}
Inequality \eqref{eqbb}  is due to Lemma \ref{pro0} with $\alpha =\frac{3}{4}$ and $\alpha = 1$ .  Now we only need to analyze the first term in \eqref{eqbb}.

We first present an upper bound for $\nu^{\mathrm{ref},k}_{h}$. Recall that $\mathbb{V}(x,y) = x^{\top}(y^2)-(x^{\top}y)^2$.
\begin{lemma}\label{lemma5}
	With probability $(1-4p)$, it holds that
	\begin{equation*}
	\nu_{h}^{\mathrm{ref},k}-\mathbb{V}(P_{s^k_{h},a^k_{h},h},V^*_{h+1})\leq 4H\beta+ \frac{6H^2SN_{0}}{n_{h}^k}+14H^2\sqrt{\frac{\iota}{n_{h}^k}}.
	\end{equation*}
\end{lemma}
\begin{proof}
	We prove by first bounding $\nu_{h}^{\mathrm{ref},k}-\frac{1}{n_{h}^k}\sum_{i=1}^{n_{h}^k}\mathbb{V}(P_{s_{h}^k,a_{h}^k,h},V_{h+1}^{\mathrm{ref},l_{i}})$. Recall that by (\ref{eqv1}),
	\begin{equation*}
	\nu_{h}^{\mathrm{ref},k}-\frac{1}{n_{h}^k}\sum_{i=1}^{n_{h}^k}\mathbb{V}(P_{s_{h}^k,a_{h}^k,h},V_{h+1}^{\mathrm{ref},l_{i}}) = -\frac{1}{n_{h}^k}(\chi_{6}+\chi_{7}+\chi_{8}),
	\end{equation*}
	where
	\begin{align} 
& \chi_{6} :=\sum_{i=1}^{n_h^k}\left(  P_{s,a,h}(V^{\mathrm{ref},l_{i}}_{h+1})^2 - (V_{h+1}^{\mathrm{ref},l_{i}}(s^{l_{i}}_{h+1 }))^2  \right), \label{eqb33-a}
\\& \chi_{7} := \frac{1}{n_h^k}\left({\sum_{i=1}^{n_h^k}V^{\mathrm{ref},l_{i}}_{h+1}(s_{h+1}^{l_{i}}  ) }\right)^2 -\frac{1}{n_h^k }\left({\sum_{i=1}^{n_h^k } P_{s,a,h}V^{\mathrm{ref},l_{i}}_{h+1} }\right)^2, \label{eqb33-b}
\\ & \chi_{8} :=   \frac{1}{n_h^k }\left({\sum_{i=1}^{ n_h^k } P_{s,a,h}V^{\mathrm{ref},l_{i}}_{h+1} }\right)^2  - \sum_{i=1}^{n_h^k }\left(P_{s,a,h}V^{\mathrm{ref},l_{i}}_{h+1}\right)^2 . \label{eqb33-c}
\end{align}
	
	By Azuma's inequality, with probability $(1-2p)$ it holds that 
	   \begin{align}
      & |\chi_{6}|\leq H^2\sqrt{2n_h^k \iota} , \nonumber \\&
       |\chi_{7}|\leq 2H|\sum_{i=1}^{n_h^k}V^{\mathrm{ref},l_i}_{h+1}(s^{l_i}_{h+1})-\sum_{i=1}^{n_h^k}P_{s_h^k,a_h^k,h}V^{\mathrm{ref},l_i} | \leq 2H^2\sqrt{2n_h^k \iota} . \nonumber
	   \end{align}
	 It left us to handle $-\chi_{8}$. By Azuma's inequality and the fact that $V^{\mathrm{ref},k}\geq V^{\mathrm{REF}}$ for any $k$, with probability $(1-p)$ it holds that
	\begin{align}
	-\chi_{8} &=  \sum_{i=1}^{n_{h}^k}\left(P_{s_h^k,a_h^k,h}V^{\mathrm{ref},l_{i}}_{h+1}\right)^2 - \frac{1}{n_h^k}\left({\sum_{i=1}^{n_{h}^k} P_{s_h^k,a_h^k,h}V^{\mathrm{ref},l_{i}}_{h+1} }\right)^2  \nonumber\\
	&\leq  \sum_{i=1}^{n_{h}^k}\left(P_{s_h^k,a_h^k,h}V^{\mathrm{ref},l_{i}}_{h+1}\right)^2 - \frac{1}{n_h^k}\left({\sum_{i=1}^{n_{h}^k} P_{s_h^k,a_h^k,h}V^{\mathrm{REF}}_{h+1} }\right)^2  \nonumber\\
	&=  \sum_{i=1}^{n_{h}^k}\left(\left(P_{s_h^k,a_h^k,h}V^{\mathrm{ref},l_{i}}_{h+1}\right)^2 - (P_{s_h^k,a_h^k,h}V^{\mathrm{REF}}_{h+1})^2 \right) \nonumber\\
	& \leq 2H^2 \sum_{i=1}^{n_{h}^k}P_{s_{h}^k,a_{h}^k,h}\lambda_{h+1}^{l_{i}} \nonumber
	\\ &  = 2H^2 \left( \sum_{i=1}^{n_{h}^k}\lambda_{h+1}^{l_{i}}(s_{h+1}^{l_{i}}   ) + \sum_{i=1}^{n_{h}^k}(P_{s_{h}^k,a_{h}^k,h   }-\textbf{1}_{s^{l_{i}}_{h+1}  } )\lambda_{h+1}^{l_{i}}      \right) \nonumber
	\\ & \leq 2H^2SN_{0} +3H^2 \sqrt{n_{h}^k\iota}.
	\end{align}
	Then we obtain that
	\begin{equation}\label{eq:bd_nu}
	\nu_{h}^{\mathrm{ref},k}-\frac{1}{n_{h}^k}\sum_{i=1}^{n_{h}^k}\mathbb{V}(P_{s_{h}^k,a_{h}^k,h},V_{h+1}^{\mathrm{ref},l_{i}})\leq   8H^2\sqrt{\frac{\iota}{n_{h}^k}}+\frac{2H^2SN_{0}}{n_{h}^k}.
	\end{equation}
	When \eqref{eq:bd_nu} holds, we have that with probability $(1-p)$,

	\begin{align}
	&\nu_{h}^{\mathrm{ref},k}- \mathbb{V}(P_{s^k_{h},a^k_{h},h}, V^*_{h+1})  \nonumber\\&  = \frac{1}{n_{h}^k}\sum_{i=1}^{n_{h}^k}\big(\mathbb{V}(P_{s^k_{h},a^k_{h},h},V_{h+1}^{\mathrm{ref},l_{i}})- \mathbb{V}(P_{s^k_{h},a^k_{h},h},V^*_{h+1})    \big)+ \big(\nu_{h}^{\mathrm{ref},k}-\frac{1}{n_{h}^k}\sum_{i=1}^{n_{h}^k}\mathbb{V}(P_{s_{h}^k,a_{h}^k,h},V_{h+1}^{\mathrm{ref},l_{i}})\big) \nonumber
	\\ & \leq   \frac{1}{n_{h}^k}\sum_{i=1}^{n_{h}^k} \big(\mathbb{V}(P_{s^k_{h},a^k_{h},h},V_{h+1}^{\mathrm{ref},l_{i}})- \mathbb{V}(P_{s^k_{h},a^k_{h},h},V^*_{h+1})    \big)         + 8H^{2}\sqrt{\frac{\iota}{n_{h}^k}} + \frac{2H^2SN_{0}}{n_{h}^k}  \nonumber
	\\ & \leq \frac{4H}{n_{h}^k}\sum_{i=1}^{n_{h}^k}P_{s^k_{h},a^k_{h},h}(V^{\mathrm{ref},l_{i}}_{h+1}-V^*_{h+1})   + 8H^{2}\sqrt{\frac{\iota}{n_{h}^k}}+ \frac{2H^2SN_{0}}{n_{h}^k}  \nonumber
	\\ & \leq  \frac{4H}{n_{h}^k}\sum_{i=1}^{n_{h}^k}(V_{h+1}^{\mathrm{ref},l_{i}}(s_{h+1}^{l_{i}})- V^{*}_{h+1}(s_{h+1}^{l_{i}}) )+\frac{4H}{n_{h}^k}\sum_{i=1}^{n_{h}^k}(P_{s^k_{h},a^k_{h},h}-\textbf{1}_{s^{l_i}_{h+1}})(V^{\mathrm{ref},l_{i}}_{h+1}-V^*_{h+1}) \nonumber 
	\\& \qquad + 8H^{2}\sqrt{\frac{\iota}{n_{h}^k}} + \frac{2H^2SN_0}{n_h^k}      \nonumber
	\\ & \leq  \frac{4H}{n_{h}^k}\sum_{i=1}^{n_{h}^k}(V_{h+1}^{\mathrm{ref},l_{i}}(s_{h+1}^{l_{i}})- V^{*}_{h+1}(s_{h+1}^{l_{i}}) ) +  14H^2\sqrt{\frac{\iota}{n_{h}^k}}+\frac{2H^2SN_{0}}{n_{h}^k}  \label{eqstar5}
	\\ &  \leq \frac{4H}{n_{h}^k}\sum_{i=1}^{n_{h}^k}(H\lambda^{l_{i}}_{h+1}(s^{l_{i}}_{h+1} ) + \beta   )+  14H^2\sqrt{\frac{\iota}{n_{h}^k}}+\frac{2H^2SN_{0}}{n_{h}^k}    \label{eqstar6}
	\\ &\leq  4H\beta + \frac{6H^2SN_{0}}{n_{h}^k}+14H^2\sqrt{\frac{\iota}{n_{h}^k}}, \nonumber
	\end{align}
	where Inequality  \eqref{eqstar5} holds with probability $(1-p)$ by Azuma's inequality and \eqref{eqstar6} holds by Corollary~\ref{coro1} (and note that the whole proof is conditioned on the successful events of Proposition~\ref{pro1} and Lemma~\ref{lemma1}).
\end{proof}

We will also prove the following bound of the total variance.

\begin{lemma} \label{lemma:b-term-variance} With probability $(1-2p)$, it holds that 
	\begin{equation}
	\sum_{s,a,h}N_{h}^{K+1}(s,a)\mathbb{V}(P_{s,a,h},V^*_{h+1})\leq  2TH+ 3\sqrt{2H^4T\iota}.
	\end{equation}
\end{lemma}

\begin{proof} By direct calculation, with probability $(1-2p)$, it holds that 
	\begin{align}
	&\sum_{s,a,h}N_{h}^{K+1}(s,a)\mathbb{V}(P_{s,a,h},V^*_{h+1})   \nonumber\\&=\sum_{k=1}^K \sum_{h=1}^{H}\mathbb{V}(P_{s^k_h,a^k_h,h},V^*_{h+1}) \nonumber
	\\ &= \sum_{k=1}^K \sum_{h=1}^K  \nonumber \big(P_{s^k_h,a^k_h,h}(V^{*}_{h+1})^2-   ( P_{s^k_h,a^k_h,h} V^*_{h+1})^2 \big)   \nonumber
	\\ & \leq \sum_{k=1}^K \sum_{h=1}^H \big(P_{s^k_h,a^k_h,h}(V^{*}_{h+1})^2 -(V^{*}_{h}(s^k_h))^2\big)+ 2H\sum_{k=1}^K \sum_{h=1}^H  | V^*_{h}(s^k_h)- P_{s^k_h,a^k_h,h} V^*_{h+1}|  \nonumber
	\\ & \leq \sqrt{2TH^4\iota} +2H\sum_{k=1}^K \sum_{h
		=1}^H  | V^*_{h}(s^k_h)- P_{s^k_h,a^k_h,h} V^*_{h+1}|  \label{eqstar7-0}
	\\& =  \sqrt{2TH^4\iota} + 2H\sum_{k=1}^K \left( V_{1}^*(s_1^k)+\sum_{h=1}^H  ( V^*_{h}(s^k_{h+1})- P_{s^k_h,a^k_h,h} V^*_{h+1}))\right)  \label{eq:eqeq}
	\\ & \leq   \sqrt{2TH^4\iota} +2TH+  2H^2\sqrt{2T\iota}   \label{eqstar7}
	\\ & \leq 2TH+ 3H^2\sqrt{2T\iota}, \nonumber
	\end{align}
	where Inequality \eqref{eqstar7-0} holds with probability $(1-p)$ by Azuma's inequality, Equation \eqref{eq:eqeq}
holds with the fact that $ V^*_{h}(s)- P_{s,a,h} V^*_{h+1}\geq V^*_{h}(s)-Q^*_{h}(s,a)\geq 0$ for any $s,a,h$ and 	Inequality \eqref{eqstar7} holds with probability $(1-p)$ by Azuma's inequality.
\end{proof}

Combining Lemma~\ref{pro0}, Lemma~\ref{lemma5}, and Lemma~\ref{lemma:b-term-variance},  we have that with probability $(1-7p)$,
\begin{align}
&\sum_{h=1}^{H}\sum_{k=1}^{K}\sqrt{\frac{\nu_{h}^{\mathrm{ref},k}}{n_{h}^k}\iota}  \leq \sum_{h=1}^{H}\sum_{k=1}^{K}\sqrt{\frac{\mathbb{V}(P_{s^k_{h},a^k_{h},h },V^*_{h+1})  }{n_{h}^k}\iota} + \sum_{h=1}^{H}\sum_{k=1}^{K}\sqrt{ \Big(\frac{4H\beta}{n_{h}^k} +\frac{6H^2SN_{0}}{(n^k_{h})^2}+14H^2\frac{\sqrt{\iota}}{(n_{h}^k)^{\frac{3}{2}}}   \Big)\iota} \nonumber
\\ & \leq O\Big(\sum_{s,a,h}\sqrt{N^{K+1}_{h}(s,a) \mathbb{V}(P_{s,a,h},V^*_{h+1}) \iota}\nonumber \\
& \qquad +\sum_{s,a,h}\sqrt{N^{K+1}_{h}(s,a)H\beta \iota }+SAH^2\sqrt{SN_{0} \iota}\log(T)+(SA\iota)^{\frac{3}{4}}H^{\frac{7}{4}}T^{\frac{1}{4}}  \Big) \nonumber 
\\ & \leq  O\Big(\sqrt{SAH^2T\iota}+\sqrt{SAH^2\beta T\iota}+ SAH^2\sqrt{SN_{0}\iota}\log(T)+(SA\iota)^{\frac{3}{4}}H^{\frac{7}{4}}T^{\frac{1}{4}}  \Big). \label{eqb0}
\end{align}

We now  bound $\check{\nu}_{h}^k$. By Corollary \ref{coro1} (and  that the whole proof is conditioned on the successful events of Proposition~\ref{pro1} and Lemma~\ref{lemma1}),  we have that
\begin{align}
&\check{\nu}_{h}^{k}\leq  \frac{1}{\check{n}^k_{h}}\sum_{i=1}^{\check{n}^k_{h}}\big(V^{\mathrm{ref},\check{l}_{i}}_{h+1}(s_{h+1}^{\check{l}_{i}}) -V^{*}_{h+1}( s_{h+1}^{\check{l}_{i}} ) \big)^2 \nonumber
\\ & \leq  \frac{1}{\check{n}^k_{h}}\sum_{i=1}^{\check{n}^k_{h}}(H^2\lambda^{\check{l}_{i}}_{h+1}(s^{\check{l}_{i}}_{h+1}) +\beta^2) \nonumber
\\ & \leq  \frac{1}{\check{n}^k_{h}}H^2SN_{0}+\beta^2. \label{eqb3}
\end{align}
By Lemma \ref{pro0}, we obtain that  
\begin{equation}\label{eqb4}
\begin{aligned}
\sum_{h=1}^{H}\sum_{k=1}^{K}\sqrt{\frac{\check{\nu}^k_{h}}{\check{n}_{h}^{k}}\iota} & \leq  \sum_{h=1}^{H}\sum_{k=1}^{K}\Big(\sqrt{\frac{\beta^2}{\check{n}_{h}^{k}} \iota}+ \frac{\sqrt{H^2SN_{0}\iota}}{\check{n}_{h}^{k}}\Big)\leq O\Big(\sqrt{SAH^3\beta^2T\iota} +SAH^3\sqrt{SN_{0}\iota}\log(T) \Big).
\end{aligned}
\end{equation}
The proof is completed by combining  \eqref{eqbb}, \eqref{eqb0}, and \eqref{eqb4}.

\end{proof}
\subsubsection{Putting Everything Together}
Recall that $\beta= \frac{1}{\sqrt{H}}$, and $N_0 = \frac{c_4 SAH^5\iota}{\beta^2} = O(SAH^6 \iota)$. Combining \eqref{eq:deflam}, Lemma~\ref{lemma:bd_psi}, Lemma~\ref{lemma:bd_xi}, Lemma~\ref{lemma:bd_phi} and Lemma~\ref{lemma:bd_b}, we conclude that with probability at least $(1-  O(H^2T^4p))$,
\begin{align}
& \sum_{k=1}^K\sum_{h=1}^K \Lambda_{h+1}^k \nonumber \\
&\leq O(\log(T)) \cdot (H^2SN_0+ H \sqrt{T \iota }) + O(H^2 \sqrt{SAT \iota})+ O(\sqrt{H^2T \iota}) \nonumber \\ 
& \qquad + O\Big(\sqrt{SAH^2T\iota} +\sqrt{SAH^2\beta T\iota}+SAH^3\sqrt{SN_0 \iota}\log(T) \nonumber
        \\& \qquad \qquad         +\sqrt{SAH^3\beta^2  T\iota}+ (SA\iota)^{\frac{3}{4}}H^{\frac{5}{2}}T^{\frac{1}{4}}   \Big). \nonumber\\
& =  O\Big(\sqrt{SAH^2T\iota} + H\sqrt{T\iota} \log (T) + \sqrt{SAH^2\beta T\iota}+SAH^3\sqrt{SN_0 \iota}\log(T) \nonumber
        \\& \qquad \qquad         +\sqrt{SAH^3\beta^2  T\iota} +  (SA\iota)^{\frac{3}{4}}H^{\frac{5}{2}}T^{\frac{1}{4}}  + H^2SN_0 \log(T)  \Big) \nonumber \\
& =  O\Big(\sqrt{SAH^2T\iota} + H\sqrt{T\iota} \log (T) +SAH^3\sqrt{SN_0 \iota}\log(T) \nonumber
              +  (SA\iota)^{\frac{3}{4}}H^{\frac{5}{2}}T^{\frac{1}{4}}  + H^2SN_0 \log(T)  \Big) \nonumber \\
        &= O\Big(\sqrt{SAH^2T\iota} + H\sqrt{T\iota} \log(T) + S^2 A^{\frac32}H^6 \iota \log(T) + (SA\iota)^{\frac34} H^{\frac52}T^{\frac14} +S^2AH^8 \iota \log(T)   \Big) \nonumber \\
        &= O\Big(\sqrt{SAH^2T\iota} + H\sqrt{T\iota} \log(T)  + S^2 A^{\frac32} H^8 \iota T^{\frac14} \Big).
\end{align}


\section{Other Results}\label{App.C}

\subsection{Local Switching Cost Analysis }
 The notion of local switching cost for RL is introduced in \citep{bai2019provably} to quantify the adaptivity of the learning algorithms. With a slight abuse of notations, we use $\pi_{k,h}$ to denote the policy at the $h$-th step of the $k$-th episode. 
 We first recall formal definition of the local switching cost.
\begin{definition} The local  switching cost at $(s,h)$ is defined as 
\begin{equation*}
n_{\mathrm{switch}}(s,h): = \sum_{k=1}^{K-1} \mathbb{I}\left[\pi_{k,h}(s) \neq  \pi_{k+1,h}(s)  \right].
\end{equation*}
The total local switching cost is then defined as 
\begin{equation*}
N_{\mathrm{switch}} := \sum_{s\in \mathcal{S}}\sum_{h=1}^H n_{\mathrm{switch}}(s,h).
\end{equation*}
\end{definition}

Now we prove Theorem \ref{thm:sw-cost}.
\begin{proof}[Proof of Theorem \ref{thm:sw-cost}]
By the definition of $e_{i}$, it is easy to verify that  $e_{i+1}\geq (1+\frac{1}{2H})e_{i}$  for any $i\geq 1$. Then the number of stages of $(s,a,h)$ is at  most 
\[
\frac{\log(\frac{N^{K+1}_h(s,a)}{2H} +1 )}{\log(1+\frac{1}{2H})}\leq 4H \log(\frac{N^{K+1}_h(s,a)}{2H} +1 ) .
\]
Because $\pi_{k,h}(s) = \arg\max_{a}Q^k_h(s,a)$, we have that
$$\mathbb{I}\left[ \pi_{k,h}(s)\neq \pi_{k+1,h}(s)  \right]\Longrightarrow \mathbb{I}\left[ \exists a, Q^{k+1}_h(s,a)\neq Q^{k}_h(s,a)  \right].$$
Now, by definition, we have that
\begin{align}
&n_{\mathrm{switch}}(s,h) = \sum_{k=1}^{K-1}\mathbb{I}\left[ \pi_{k,h}(s)\neq \pi_{k+1,h}(s)  \right] \nonumber 
\\& \leq \sum_{k=1}^{K-1}\mathbb{I}\left[ \exists a, Q^{k+1}_h(s,a)\neq Q^{k}_h(s,a)  \right] \nonumber
\\ & \leq \sum_{a}4H \log(\frac{N^{K+1}_h(s,a)}{2H} +1 )\nonumber .
\end{align}
Finally, by the concavity of $\log(x)$ in $x$, 
  the total local switching cost of \UCBADV is bounded by 
\begin{align}
&N_{\mathrm{switch}} = \sum_{s\in \mathcal{S}}\sum_{h=1}^H n_{\mathrm{switch}}(s,h)  \nonumber
\\ &  \leq \sum_{s,a,h}4H \log(\frac{N^{K+1}_h(s,a)}{2H} +1 )  \nonumber
\\ & \leq 4H^2SA\log(\frac{T}{2SAH^2}+1) \nonumber \\ & = O(H^2SA\log(\frac{K}{SAH})).\nonumber 
\end{align}
\end{proof}

\subsection{Application to Concurrent RL}
In concurrent RL,  multiple agents act in parallel and shares the experience in a limited way to accelerate the learning process. In this subsection, we follow the setting in \citep{bai2019provably} to introduce the problem. 

Suppose there are $M$ parallel agents, where each agent interacts with the environments independently. In the concurrent RL problem, each agent finishes an episode simultaneously, so that there are $M$ episodes done per concurrent round. The agents can only exchange experience and update their policies at the end of each round. The goal is to find an $\epsilon$-optimal policy using the minimum number of rounds, which we also refer to as the number of concurrent episodes. 

In Algorithm \ref{alg2}, we present the details of the concurrent \UCBADV algorithm. The idea is to simulate the single-agent \UCBADV by treating the $M$ episodes finished in a single round as $M$ consecutive episodes (without policy change) in the single-agent setting. We collect the trajectories and feed them to the single-agent \UCBADV. When an update is triggered in the single-agent \UCBADV during an episode, we update the $Q$-function (as well as the value function) and discard the  trajectories left in the round.

\begin{algorithm}[tb]
	\caption{Concurrent \UCBADV}
	\begin{algorithmic}\label{alg2}
		\STATE{\textbf{Initialize:} $Q_{h}(s,a) \leftarrow H-h+1$, $k\leftarrow 1$ , $K_{\epsilon} \leftarrow \frac{c_5 SAH^3\log(\frac{SAH}{\epsilon})}{\epsilon^2M}$ ($c_5$ is a large enough universal constant).}
		\FOR{concurrent episodes $k=1,2,3,\dots$}
		\STATE{ All agents follow the same policy $\pi_{k}$ where $\pi_{k,h}(s) = \arg\max_{a}Q_{h}(s,a)$.}
		\FOR{$i=1,2,3,\dots,M$}
		\STATE{Collect the trajectory of the $i$-th agent and feed it to \UCBADV}
		\IF{an update is triggered}
		\STATE{Update  $Q$-value function following \UCBADV;}
		\STATE{\textbf{break}}
		\ENDIF
		\ENDFOR
		\IF{The number of trajectories use is greater than or equal to $K_{\epsilon}$}
		\STATE{\textbf{break}}
		\ENDIF
		\ENDFOR
	\end{algorithmic}
\end{algorithm}

We now prove Corollary \ref{cor:concurrent-RL} that shows the performance of the concurrent \UCBADV.

\begin{proof}[Proof of Corollary \ref{cor:concurrent-RL}]
The proof follows the similar lines in the proof of Theorem 5 in \citep{bai2019provably}. By Theorem \ref{thm:sw-cost}, the switching cost is at most $O(H^2SA\log(\frac{K_{\epsilon}}{SAH}))$, so there are at most
\begin{equation*}
O(H^2SA\log(\frac{K_{\epsilon}}{SAH}) +\frac{K_{\epsilon}}{M} ) = \tilde{O}(H^2SA+\frac{H^3SA}{\epsilon^2M})
\end{equation*}
concurrent episodes. On the other hand, the regret incurred in the episodes corresponding to $K_{\epsilon}$ is at most $\tilde{O}(\sqrt{SAH^3K_{\epsilon}})\leq K_{\epsilon}\epsilon$,  so by randomly choosing an episode index $k$ and selecting $\pi=  \pi_{k}$ we achieve a policy with expected performance at most $\epsilon$ below the optimum.
\end{proof}

\subsection{Lower Bound of the Sample Complexity}
\begin{thm}\label{thmlb1}
For any $H$, $S$, and $A$ greater than a universal constant, and all $\epsilon\in (0,\frac{8}{H}]$, for any algorithm with input parameter $\epsilon$, there exists an episodic MDP with $S$ states, $A$ actions, horizon $H$ such that, with probability at least $1/2$, among the execution history of the algorithm, there are at least 
$\Omega({SAH^3}/{\epsilon^2})$
episodes in which the corresponding policy $\pi_k$ satisfies that $V^*_{1}(s^k_1)-V^{\pi_{k}}_{1}(s^k_1)>\epsilon$.
\end{thm}
\begin{proof}[Proof Sketch.]
	Instead of presenting a concrete proof of Theorem \ref{thmlb1}, we provide the high-level intuition in the construction and analysis.
	
Like the regret lower bound analysis in \citep{jin2018q}, we consider the special case where $S=A=2$. It does not require too much difficulty to generalize to arbitrary $S$ and $A$. Also, we will use almost the same hard instance as constructed in the proof of  Theorem 3 in \citep{jin2018q}. 

We recall the structure of ``JAO MDP'' in \citep{jaksch2010near}. There are two states in the MDP, named $s_{0}$ and $s_{1}$. The rewards are defined as  $r(s_{0},a) = 0$ and $r(s_{1},a) = 1$ for any $a$ and the transition probabilities are defined as $P(\cdot|s_{1},a) = [\delta, 1-\delta]^{\top}, \forall a$, $P(\cdot|s_{0},a) = [1-\delta,\delta]^{\top}, \forall a\neq a^*$ and $P(\cdot|s_{0},a^*) = [1-\delta-\epsilon,\delta+\epsilon]^{\top}$. Clearly the optimal action for state $s_{0}$ is $a^*$.  Let $\delta<\frac{1}{2}$ be fixed. By the lower bound of \citep{jaksch2010near}, there exists a constant $c_5>0$, such that  for any $\epsilon\in (0,\frac{\delta}{2})$, it costs  at least $c_5 \cdot \frac{\delta}{\epsilon^2}$ observations to identify $a^*$ with non-trivial probability.

By connecting $H$ JAO MDPs with different optimal actions layer by layer, we get an episodic MDP with horizon $H$. We choose $\delta = \frac{16}{H}$ to ensure that the MDP is well-mixed for $h\geq \frac{H}{2}$. 
For any $\epsilon\leq \frac{8}{H} = \frac{\delta}{2}$ and $h\geq \frac{H}{2}$, the agent reaches $s_{0}$  in the $h$-th layer with at least constant probability. 
If  there are at least $\frac{7H}{8}$ layers in which the agent can not identify $a^*$, then the agent makes $\Omega(H)$ mistakes in the range $h\in [\frac{H}{2},\frac{3H}{4}]$. Because each mistake for $h\in [\frac{H}{2},\frac{3H}{4}]$  leads to $\Omega(\epsilon H)$ regret , the expected regret incurred during one episode is $\Omega(\epsilon H^2)$.  As a result, if the total number of observations is less than $\frac{c_5H}{8}\cdot \frac{\delta}{\epsilon^2}$ (i.e., number of episodes less than $\frac{c_5}{8} \cdot \frac{\delta}{\epsilon^2}$), the expected regret per episode is $\Omega(\epsilon H^2)$. Replacing $\epsilon$ by $\epsilon H^2$, we have that for the first $\Theta(\delta H^4 / \epsilon^2) = \Theta(H^3 / \epsilon^2) $ episodes, the expected regret per episode is $\Omega(\epsilon)$. The proof is then completed by applying Markov's inequality.

\end{proof}

\end{document}